\documentclass{article}
\usepackage{proceed}
\usepackage[margin=1in]{geometry}
\usepackage{times}
\usepackage{helvet}
\usepackage{courier}
\usepackage{pdfsync}

\usepackage{graphicx} 
\usepackage{subfigure}
\usepackage{sidecap}

\newcommand{\figwidththree}{0.32\textwidth}
\newcommand{\figwidthsidetwo}{0.33\textwidth}
\newcommand{\figwidthfour}{0.235\textwidth}

\usepackage{amsmath, amssymb}
\usepackage{amsthm}
\usepackage{natbib}

\usepackage{algorithm}
\usepackage{algorithmic}

\input{Definitions.tex}

\newcommand{\xdim}{d}
\newcommand{\rdim}{k}

\newcommand{\sampiter}{t}
\newcommand{\stepsize}{\alpha}
\newcommand{\regwgt}{\eta}

\newcommand{\Actions}{\mathcal{A}}
\newcommand{\States}{\mathcal{S}}
\newcommand{\Pfcn}{\mathrm{Pr}}

\newcommand{\Rfcn}{r}

\newcommand{\fa}{x}
\newcommand{\Sset}{\mathcal{S}}

\newcommand{\Amat}{\mathbf{A}}
\newcommand{\pAmat}{\tilde{\mathbf{A}}}
\newcommand{\Bmat}{\mathbf{B}}

\newcommand{\Dmat}{\mathbf{D}}
\newcommand{\Hmat}{\mathbf{H}}

\newcommand{\Pmat}{\mathbf{P}}

\newcommand{\Qmat}{\mathbf{Q}}
\newcommand{\Smat}{\mathbf{S}}

\newcommand{\Umat}{\mathbf{U}}
\newcommand{\Vmat}{\mathbf{V}}

\newcommand{\defeq}{\mathrel{\overset{\makebox[0pt]{\mbox{\normalfont\tiny\sffamily def}}}{=}}}
\newcommand{\Lambdamat}{\boldsymbol{\Lambda}}

\newcommand{\Sigmamat}{\boldsymbol{\Sigma}}

\newcommand{\avec}{\mathbf{a}}
\newcommand{\bvec}{\mathbf{b}}
\newcommand{\pbvec}{\tilde{\mathbf{b}}}
\newcommand{\cvec}{\mathbf{c}}
\newcommand{\dvec}{\mathbf{d}}
\newcommand{\evec}{\mathbf{e}}
\newcommand{\gvec}{\mathbf{g}}
\newcommand{\hvec}{\mathbf{h}}

\newcommand{\uvec}{\mathbf{u}}
\newcommand{\vvec}{\mathbf{v}}
\newcommand{\wvec}{\mathbf{w}}
\newcommand{\xvec}{\mathbf{x}}

\newcommand{\E}{\mathbb{E}}
\newcommand{\eigmax}{\lambda_{\text{max}}}

\title{Effective sketching methods for value function approximation}

\author{Yangchen Pan, Erfan Sadeqi Azer and Martha White\\
Department of Computer Science\\
Indiana University Bloomington\\
yangpan@iu.edu, esadeqia@iu.edu, martha@indiana.edu} 

\begin{document}

\maketitle

\begin{abstract}
High-dimensional representations, such as radial basis function networks or tile coding, are common choices for policy evaluation in reinforcement learning. Learning with such high-dimensional representations, however, can be expensive, particularly for matrix methods, such as least-squares temporal difference learning or quasi-Newton methods that approximate matrix step-sizes. In this work, we explore the utility of sketching for these two classes of algorithms. We highlight issues with sketching the high-dimensional features directly, which can incur significant bias. As a remedy, we demonstrate how to use sketching more sparingly, with only a left-sided sketch, that can still enable significant computational gains and the use of these matrix-based learning algorithms that are less sensitive to parameters. We empirically investigate these algorithms, in four domains with a variety of representations. Our aim is to provide insights into effective use of sketching in practice.
\end{abstract}

\section{INTRODUCTION}

A common strategy for function approximation in reinforcement learning is
to overparametrize: generate a large number of features to provide
a sufficiently complex function space. 
For example, one typical representation is a radial basis function network, 
where the centers for each radial basis function are chosen to exhaustively cover the
observation space.
Because the environment is unknown---particularly for the incremental learning setting---such an overparameterized representation 
is more robust to this uncertainty because a reasonable representation is guaranteed for any part
of the space that might be visited.
Once interacting with the environment, however, it is likely not all features
will become active, and that a lower-dimensional subspace will be visited. 

A complementary approach for this high-dimensional representation expansion in reinforcement learning,
therefore, is to use projections. In this way, we can overparameterize for robustness,
but then use a projection to a lower-dimensional space to make learning feasible.
For an effectively chosen projection, we can avoid discarding important information,
and benefit from the fact that the agent only visits a lower-dimensional subspace of the environment
in the feature space.

Towards this aim, we investigate the utility of sketching: projecting with a random matrix. 
Sketching has been extensively used for efficient communication
and solving large linear systems, 
with a solid theoretical foundation and a variety of different sketches \citep{woodruff2014sketching}.
%
Sketching has been previously used in reinforcement learning, specifically to reduce the dimension of the features. 
\citet{bellemare2012sketch} replaced the standard biased hashing function used for tile coding \cite{sutton1996generalization}, instead using count-sketch.\footnote{They called the sketch the tug-of-war sketch, but it is more standard to call it count-sketch.}
\citet{ghavamzadeh2010lstd} investigated sketching features to reduce the dimensionality
and make it feasible to run least-squares temporal difference learning (LSTD) for policy evaluation. 
In LSTD, the value function is estimated by incrementally computing a $\xdim\times\xdim$ matrix $\Amat$, where $\xdim$ is the number of features, and an $\xdim$-dimensional vector $\bvec$, where the parameters are estimated as the solution to this linear system.
Because $\xdim$ can be large, they randomly project the features to reduce the matrix size to $\rdim\times\rdim$, with $\rdim \ll \xdim$. 

For both of these previous uses of sketching, however, the resulting value function estimates
are biased. This bias, as we show in this work, can be quite significant, resulting in significant
estimation error in the value function for a given policy. As a result, any gains from using
LSTD methods---over stochastic temporal difference (TD) methods---are largely overcome by this bias. 
A natural question is if we can benefit from sketching, 
with minimal bias or without incurring any bias at all.

In this work, we propose to instead sketch the linear system
in LSTD. The key idea
is to only sketch the constraints of the system (the left-side of $\Amat$) rather than the variables
(the right-side of $\Amat$).
Sketching features, on the other hand, by design, sketches both constraints and variables. 
We show that even with a straightforward linear system solution,
the left-sided sketch can significantly reduce bias. 
We further show how to use this left-sided sketch within a quasi-Newton algorithm,
providing an unbiased policy evaluation algorithm that can still benefit from the computational improvements of sketching.

The key novelty in this work is designing such system-sketching algorithms
when also incrementally computing the linear system solution. 
There is a wealth of literature on sketching linear systems, to reduce computation.
In general, however, many sketching approaches cannot be applied
to the incremental policy evaluation problem, because
the approaches are designed for a static linear system. 
For example, \citet{gower2015randomized} provide a host of possible
solutions for solving large linear systems. However, they assume access to $\Amat$ upfront, 
so the algorithm design, in memory and computation, is not suitable
for the incremental setting. 
Some popular sketching approaches, such as Frequent Directions \citep{ghashami2014improved},
has been successfully used for the online setting, for quasi-Newton algorithms \citep{luo2016efficient};
however, they sketch symmetric matrices, that are growing with number of samples.

This paper is organized as follows.
We first introduce the policy evaluation problem---learning a value function for a fixed policy---and provide background
on sketching methods. We then illustrate issues with only sketching features, in terms of quality of the value function approximation.
We then introduce the idea of using asymmetric sketching for policy evaluation with LSTD,
and provide an efficient incremental algorithm that is $O(\xdim \rdim)$ on each step. 
We finally highlight settings where we expect sketching to perform particularly well in practice,
and investigate the properties of our algorithm on four domains, and with a variety of representation properties. 



\section{PROBLEM FORMULATION}

We address the policy evaluation problem within reinforcement learning,
where the goal is to estimate the value function for a given policy\footnote{To focus the investigation on sketching, we consider
the simpler on-policy setting in this work. Many of the results, however, generalize to the off-policy setting,
where data is generated according to a behavior policy different than the given target policy we wish to evaluate.}. 
As is standard, the agent-environment interaction is formulated as a Markov decision process $(\States, \Actions, \Pfcn, \Rfcn)$,
where 
$\States$ is the set of states, 
$\Actions$ is the set of actions, and 
$\Pfcn: \States \times \Actions \times \States \rightarrow [0,\infty)$ is the one-step state transition dynamics. 
On each time step $t = 1,2,3,...$, the agent selects an action according to its policy $\pi$, $A_t\sim\pi(S_t, \cdot)$,
with $\pi: \States \times \Actions \rightarrow [0, \infty)$ and transitions into a new state $S_{t+1}\sim\Pfcn(S_t, A_t, \cdot)$ and obtains scalar reward $R_{t+1} \defeq \Rfcn(S_t,A_t,S_{t+1})$. 

For policy evaluation, the goal is to estimate the value function, $v_\pi: \States \rightarrow \mathbb{R}$, which corresponds to the expected return when following policy $\pi$ 
\begin{equation*}
v_\pi(s) \defeq \mathbb{E}_\pi[G_t | S_t = s], 
\end{equation*} 
where $\mathbb{E}_\pi$ is the expectation over future states when selecting actions according to  $\pi$. The return, $G_t \in \mathbb{R}$ is the discounted sum of future rewards given actions are selected according to $\pi$:
\begin{align}
G_t &\defeq R_{t+1} + \gamma_{t+1} R_{t+2} + \gamma_{t+1}\gamma_{t+2} R_{t+3} + ... \\
&= R_{t+1} + \gamma_{t+1} G_{t+1} \nonumber
\end{align} 
where $\gamma_{t+1}\in[0,1]$ is a scalar that depends on $S_t, A_t, S_{t+1}$ and discounts the contribution of future rewards exponentially with time. A common setting, for example, is a constant discount. This recent generalization to state-dependent discount \citep{sutton2011horde,white2016unifying} enables either episodic or continuing problems, and so we adopt this more general formalism here.
 
 We consider linear function approximation to estimate the value function. In this setting, 
 the observations are expanded to a higher-dimensional space, such as through tile-coding, radial basis functions or Fourier basis. 
 Given this nonlinear encoding $\fa: \States \rightarrow \mathbb{R}^\xdim$, the value is approximated as
 $v_\pi(S_t) \approx  \wvec^\top \xvec_t$ for $\wvec\in\mathbb{R}^d$ and $\xvec_t \defeq \fa(S_t)$. 

One algorithm for estimating $\wvec$ is least-squares temporal difference learning (LSTD). 
The goal in LSTD($\lambda$) \citep{boyan1999least} is to minimize the mean-squared projected Bellman error, which can be represented
as solving the following linear system
\begin{align*}
\Amat &\defeq \E_\pi[\evec_\sampiter (\xvec_\sampiter-\gamma_{\sampiter+1}\xvec_{\sampiter+1})^\top]\\
\bvec &\defeq \E_\pi[R_{\sampiter+1}\evec_\sampiter] 
.
\end{align*}
where $\evec_\sampiter \defeq \gamma_{t+1} \lambda \evec_{\sampiter-1} + \xvec_t$ is called the eligibility trace
for trace parameter $\lambda \in [0,1]$. 
To obtain $\wvec$, the system $\Amat$ and $\bvec$ are incrementally estimated, to
 solve $\Amat \wvec = \bvec$.
For a trajectory $\{(S_\sampiter,A_\sampiter,S_{\sampiter+1},R_{\sampiter+1})\}_{\sampiter=0}^{T-1}$, let $\dvec_t \defeq \xvec_\sampiter - \gamma_{\sampiter+1} \xvec_{\sampiter+1}$, then the above two expected terms are usually computed via sample average
that can be recursively computed, in a numerically stable way, as
\begin{align*}
\Amat_{t+1} &= \Amat_{t} + \frac{1}{t+1} \left( \evec_{t} \dvec_t^\top - \Amat_{t}\right)\\
\bvec_{t+1} &= \bvec_{t} + \frac{1}{t+1} \left(\evec_{t} R_{t+1} - \bvec_{t} \right)
\end{align*}
with $\Amat_0 = \zerovec$ and $\bvec_0  = \zerovec$. The incremental estimates $\Amat_t$ and $\bvec_t$
converge to $\Amat$ and $\bvec$. A naive algorithm, where $\wvec = \Amat_t^\inv \bvec_t$ is recomputed on each step, would result in $\mathcal{O}(\xdim^3)$ computation to compute the inverse $\Amat_t^\inv$. 
Instead, $\Amat_t^\inv$ is incrementally updated using the Sherman-Morrison formula, with $\Amat_0^\inv = \xi$ for a small $\xi>0$
\begin{align*}
\Amat_t^\inv &= \left(\frac{t-1}{t} \Amat_{t-1} + \frac{1}{t} \evec_{t} \dvec_t^\top\right)^\inv \\
&= \frac{t}{t-1}\left(\Amat_{t-1}^\inv +  \frac{ \Amat_{t-1}^\inv \evec_{t} \dvec_t^\top\Amat_{t-1}^\inv}{t-1+ \dvec_t^\top\Amat_{t-1}^\inv \evec_{t} }\right)
\end{align*}
 requiring 
 $\mathcal{O}(d^2)$ storage and computation per step.
Unfortunately, this quadratic cost is prohibitive for many incremental learning settings. In our experiments, even $\xdim=10,000$ was prohibitive, since $\xdim^2 = 100$ million. 

A natural approach to improve computation to solve for $\wvec$ is to use stochastic methods, 
such as TD($\lambda$) \citep{sutton1988learning}. This algorithm incrementally updates with $\wvec_{t+1} = \wvec_t + \stepsize \delta_t \evec_t$ for stepsize $\stepsize > 0$ and TD-error $\delta_t = R_{t+1} + (\gamma_{t+1} \xvec_{t+1} - \xvec_t)^\top \wvec_t$. The expectation of this update is $\E_\pi[\delta_t \evec_t] = \bvec - \Amat \wvec_t$; the fixed-point solutions are the same for both LSTD and TD, but LSTD corresponds to a batch solution whereas TD corresponds to a stochastic update. 
Though more expensive than TD---which is only $O(\xdim)$---LSTD does have several advantages. 
Because LSTD is a batch method, it summarizes all samples (within $\Amat$), and so can be more
sample efficient. Additionally, LSTD has no step-size parameter, using a closed-form solution for $\wvec$.   

Recently, there has been some progress in better balancing between TD and LSTD.
\citet{pan2017accelerated} derived a quasi-Newton algorithm, called accelerated gradient TD (ATD), 
giving an unbiased algorithm that has some of
the benefits of LSTD, but with significantly reduced computation because they only maintain a low-rank approximation to $\Amat$. 
The key idea is that $\Amat$ provides curvature information, and so can significantly improve step-size selection for TD and
so improve the convergence rate. The approximate $\Amat$ can still provide useful curvature information, but can be significantly cheaper to compute.
We use the ATD update to similarly obtain an unbiased algorithm, but use sketching approximations instead of low-rank approximations.
First, however, we investigate some of the properties of sketching.

\section{ISSUES WITH SKETCHING THE FEATURES}\label{sec_issues}

One approach to make LSTD more feasible is to project---sketch---the features. Sketching involves
sampling a random matrix $\Smat: \RR^{\rdim \times \xdim}$ from a family of matrices $\Sset$, to project a given $\xdim$-dimensional vector $\xvec$ to a (much smaller) $\rdim$-dimensional vector $\Smat \xvec$. The goal in defining this class of sketching matrices
is to maintain certain properties of the original vector. The following is a standard definition for such a family.
\begin{definition}[Sketching]\label{defSketch} Let $\xdim$ and $\rdim$ be positive integers, $\delta\in (0,1)$, and $\epsilon \in \RR^+$. Then, $\mathcal{S}\subset \RR^{\rdim\times \xdim}$ is called a family of sketching matrices with parameters $(\epsilon,\delta)$, if for a random matrix, $\Smat$, chosen uniformly at random from this family, we have that $\forall \xvec\in \RR^\xdim$
\begin{equation*}
\PP\Big[ (1-\epsilon) \| \xvec \|_2^2 \leq \| \Smat  \xvec \|_2^2 \leq (1+\epsilon) \| \xvec \|_2^2\Big] \geq 1-\delta.
\end{equation*}
where the probability is w.r.t. the distribution over $\Smat$.  
\end{definition}
We will explore the utility of sketching the features with several
common sketches. 
These sketches all require $\rdim = \Omega(\epsilon^{-2} \ln (1/\delta) \ln d)$. 

\begin{figure*}[]
\vspace{-0.2cm}
	\subfigure[Mountain Car, tile coding]{
		\includegraphics[width=\figwidthsidetwo]{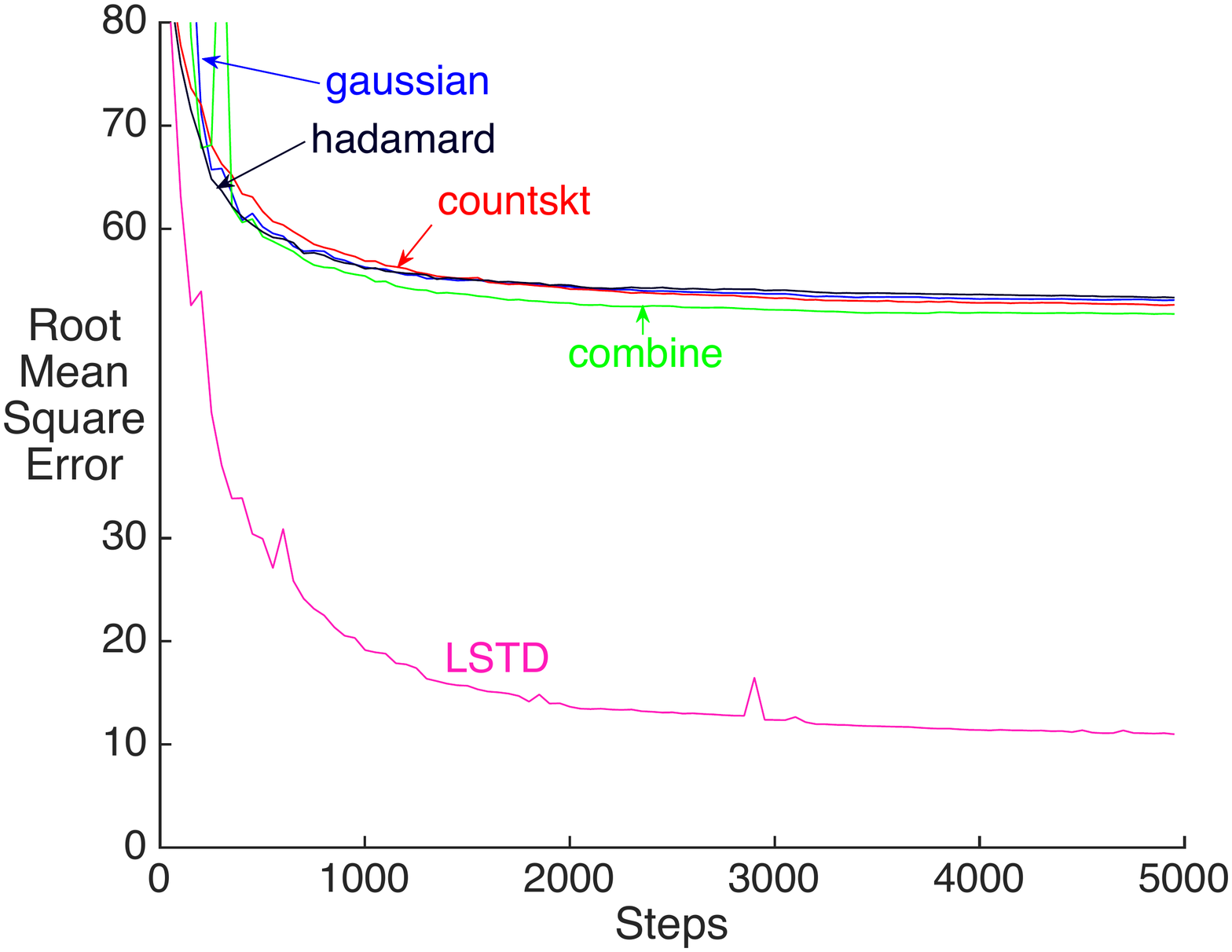}\label{fig:mcar_compskts_tile50r}}	
	\subfigure[Mountain Car, RBF]{
		\includegraphics[width=\figwidthsidetwo]{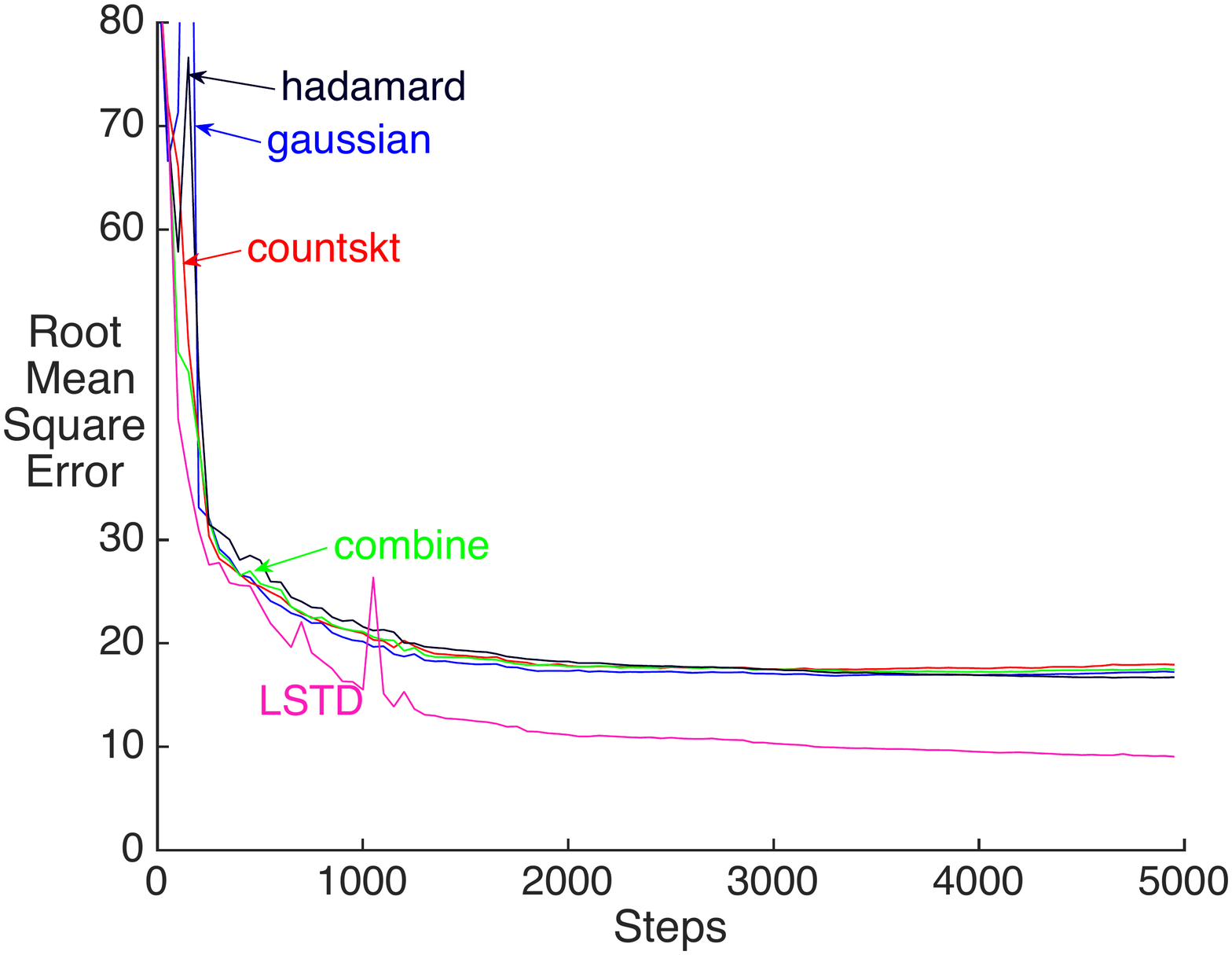}\label{fig:mcar_compskts_rbf50r}}	\\
	\subfigure[Puddle World, tile coding]{
		\includegraphics[width=\figwidthsidetwo]{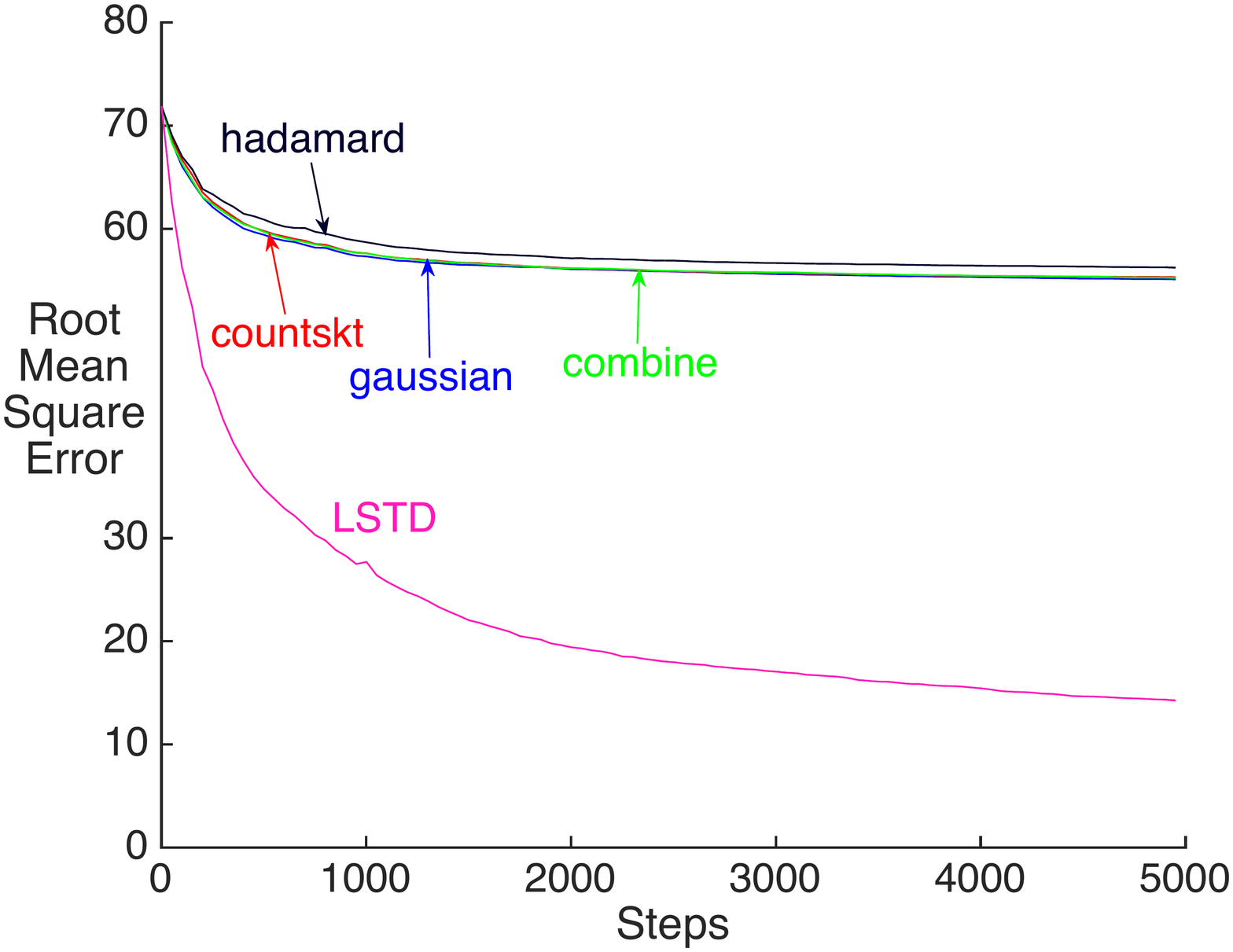}\label{fig:pd_compskts_tile50r}}	
	\subfigure[Puddle World, RBF]{
		\includegraphics[width=\figwidthsidetwo]{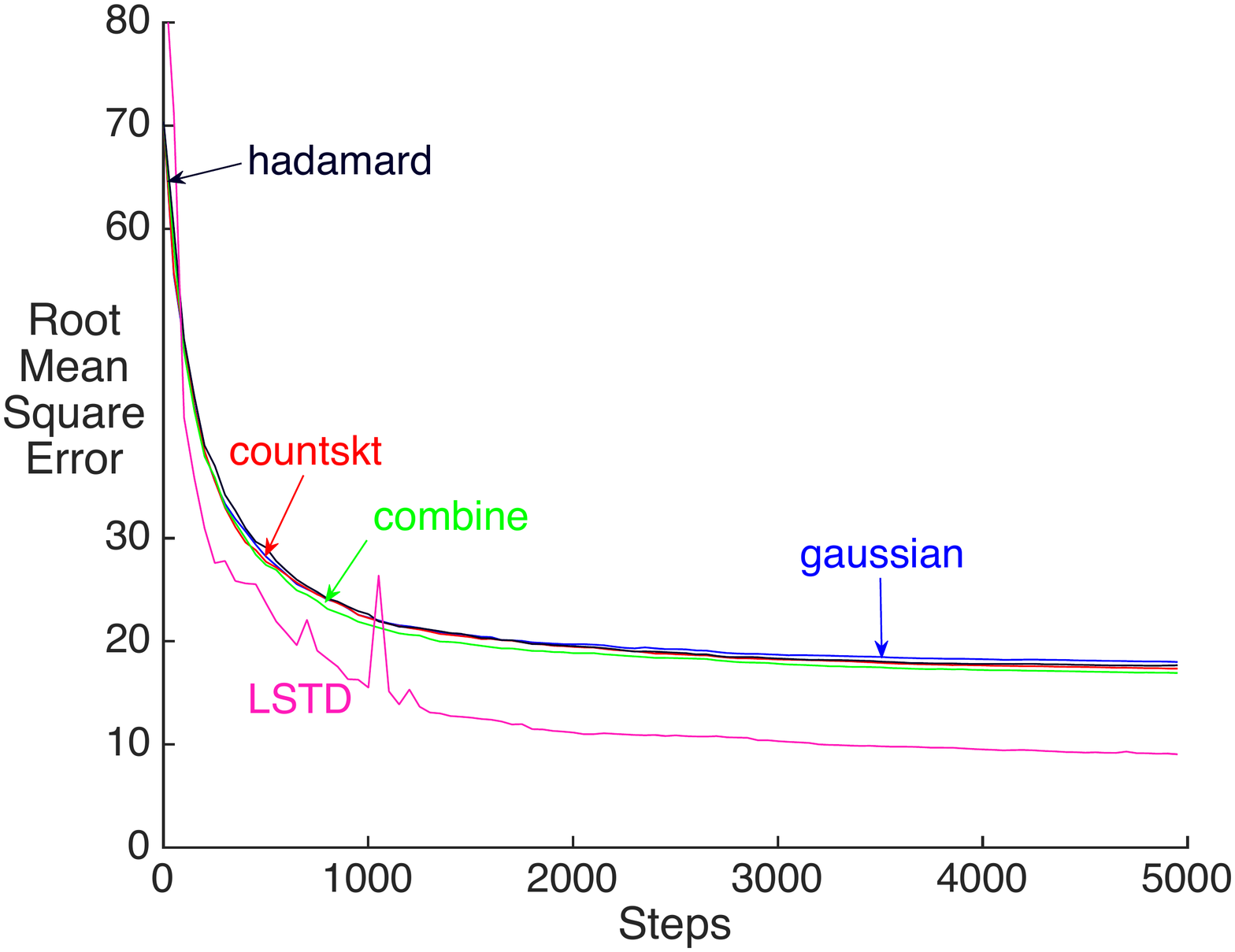}\label{fig:pd_compskts_rbf50r}}	
	\begin{minipage}{0.1cm}
		\vspace{-9cm}
		\end{minipage}
	\begin{minipage}{\figwidththree}
		\vspace{-9cm}
		\caption{ 
			Efficacy of different sketches for sketching the features for LSTD, with $\rdim = 50$. The RMSE is w.r.t. the optimal value function, computed using rollouts.  
			LSTD($\lambda$) is included as the baseline, with $\wvec = \Amat^\inv \bvec$, with the other curves corresponding to different sketches of the features, to give $\wvec = (\Smat \Amat \Smat^\top)^\inv \Smat \bvec$ as used for the random projections LSTD algorithm. The RBF width in Mountain Car is $\sigma = 0.12$ times the range of the state space and in Puddle World is $\sigma = \sqrt{0.0072}$. The $1024$ centers for RBFs are chosen to uniformly cover the 2-d space in a grid.  For tile coding, we discretize each dimension by $10$, giving $10\times10$ grids, use $10$ tilings, and set the memory size as $1024$. The bias is high for tile coding features, and much better for RBF features, though still quite large. The different sketches perform similarly. 
		}\label{fig:compare_skts50r}
	\end{minipage}
\end{figure*}

\textbf{Gaussian random projections}, also known as the JL-Transform \citep{johnson1984extensions}, has each entry in $\Smat$ i.i.d. sampled from a Gaussian, $\mathcal{N}(0, \frac{1}{\rdim})$. 

\textbf{Count sketch} selects exactly one uniformly picked non-zero entry in each column, and sets that entry to either $1$ or $-1$ with equal probability \citep{charikar2002finding,gilbert2010sparse}. 
The Tug-of-War sketch \citep{alon1996space} performs very similarly to Count sketch in our experiments, and so we omit it. 

\textbf{Combined sketch} is the product of a count sketch matrix and a Gaussian projection matrix \citep{wang2015apractical,chris2015pca}.

\textbf{Hadamard sketch}---the Subsampled Randomized Hadamard Transform---is computed as $\Smat = \frac{1}{\sqrt{\rdim \xdim}}\Dmat \Hmat_d \Pmat$, where $\Dmat \in \RR^{\xdim \times \xdim}$ is a diagonal matrix with each diagonal element uniformly sampled from $\{1, -1\}$, $\Hmat_\xdim \in \RR^{\xdim \times \xdim}$ is a Hadamard matrix and $\Pmat \in \RR^{\xdim \times \rdim}$ is a column sampling matrix \citep{nir2006fastjlt}. 

Sketching provides a low-error between the recovery $\Smat^\top \Smat \xvec$ and the original $\xvec$, with high probability.
For the above families, the entries in $\Smat$ are zero-mean i.i.d. with variance 1,
giving $\E[\Smat^\top \Smat] = \eye$ over all possible $\Smat$. Consequently,
in expectation, the recovery $\Smat^\top \Smat \xvec$ is equal to $\xvec$. For a stronger result, 
a Chernoff bound can be used to bound the deviation of $\Smat^\top \Smat$ from this expected value: for the parameters $(\epsilon,\delta)$ of the matrix family, we get that $\PP\Big[ (1-\epsilon)I\prec \Smat^\top \Smat\prec (1+\epsilon)I\Big] \geq 1-\delta$.


These properties suggest that using sketching for the feature vectors should provide effective approximations. \cite{bellemare2012sketch} showed that they could use these projections for tile coding, rather than the biased hashing function that is typically used, to improve learning performance for the control setting. The efficacy, however, of sketching given features, versus using the unsketched features, is less well-understood. 
 
We investigate the properties of sketching the features, shown in Figure~\ref{fig:compare_skts50r} with a variety of sketches in two benchmark domains for RBF and tile-coding representations (see \cite[Chapter 8]{sutton1998reinforcement} for an overview of these representations). For both domains, the observations space is 2-dimensional, with expansion to $\xdim = 1024$ and $\rdim = 50$. The results are averaged over 50 runs, with $\xi,\lambda$ swept over 13 values, with ranges listed in Appendix \ref{app_experiments}. We see that sketching the features can incur significant bias, particularly for tile coding, even with a reasonably large $\rdim = 50$ to give $O(\xdim \rdim)$ runtimes. This bias reduces with $\rdim$, but remains quite high and so is likely too unreliable for practical use.

\section{SKETCHING THE LINEAR SYSTEM}

All of the work on sketching within reinforcement learning has investigated 
sketching the features; however, we can instead consider sketching the linear system,
$\Amat \wvec = \bvec$. For such a setting, we can sketch the left and right subspaces
of $\Amat$ with different sketching matrices, $\Smat_L \in \RR^{\rdim_L \times \xdim}$ and
$\Smat_R \in \RR^{\rdim_R \times \xdim}$. Depending on the choices of $\rdim_L$ and $\rdim_R$,
we can then solve the smaller system $\Smat_L \Amat \Smat_R^\top \Smat_R \wvec = \Smat_L \bvec$ efficiently. 
The goal is to better take advantage of the properties for the different sides of an asymmetric matrix $\Amat$.

One such natural improvement should be in one-sided sketching.
By only sketching from the left, for example, and setting $\Smat_R = \eye$,
we do not project $\wvec$. Rather, we only project the constraints to the linear
system $\Amat \wvec = \bvec$. Importantly, this does not introduce bias: the original solution
$\wvec$ to $\Amat \wvec = \bvec$ is also a solution to $\Smat\Amat \wvec = \Smat\bvec$ for any sketching
matrix $\Smat$. The projection, however, removes uniqueness in terms of the solutions $\wvec$,
since the system is under-constrained. 
Conversely, by only sketching from the right, and setting $\Smat_L = \eye$,
we constrain the space of solutions to a unique set, and do not remove any constraints.
For this setting, however, it is unlikely that $\wvec$ with $\Amat \wvec = \bvec$ satisfies $\Amat \Smat^\top \wvec = \bvec$.

The conclusion from many initial experiments is that the key benefit from asymmetric sketching is
when only sketching from the left.
We experimented with all pairwise combinations of Gaussian random projections, Count sketch, Tug-of-War sketch and Hadamard sketch for $\Smat_L$ and $\Smat_R$. We additionally experimented with only sketching from the right, setting $\Smat_L = \eye$. 
In all of these experiments, we found asymmetric
sketching provided little to no benefit over using $\Smat_L = \Smat_R$ and that sketching only
from the right also performed similarly to using $\Smat_L = \Smat_R$. We further investigated
column and row selection sketches (see \cite{wang2015apractical} for a thorough overview), but also found
these to be ineffective. We therefore proceed with an investigation into effectively using left-side sketching. 
In the next section, we provide an efficient $\mathcal{O}(\xdim \rdim)$ algorithm to compute $(\Smat_L \Amat)^\pinv$,
to enable computation of $\wvec = (\Smat_L \Amat)^\pinv \Smat_L \bvec$ and for use
within an unbiased quasi-Newton algorithm. 

We conclude this section with an interesting connection to a data-dependent projection
method that has been used for policy evaluation, that further motivates the utility of sketching
only from the left. 
This algorithm, called truncated LSTD (tLSTD) \citep{gehring2016incremental},
incrementally maintains a rank $\rdim$ approximation of $\Amat$ matrix, using
an incremental singular value decomposition. 
We show below that this approach corresponds to projecting $\Amat$ from the left with the top $\rdim$ left singular vectors. 
This is called a data-dependent projection, because the projection depends on the observed data,
as opposed to the data-independent projection---the sketching matrices---which is randomly sampled independently of the data. 
%
%
%
\begin{proposition}
Let $\Amat=\Umat\Sigmamat \Vmat^\top$ be singular value decomposition of the true $\Amat$.
Assume the singular values are in decreasing order and let $\Sigmamat_\rdim$
be the top $\rdim$ singular values, with corresponding $\rdim$ left singular vectors $\Umat_\rdim$ and $\rdim$ right singular vectors $\Vmat_\rdim$.  
Then the solution $\wvec = \Vmat_\rdim \Sigmamat_\rdim^\pinv \Umat_\rdim^\top \bvec$ (used for tLSTD) corresponds
to LSTD using asymmetric sketching with $\Smat_L=\Umat_\rdim^\top$ and $\Smat_R=\eye$. 
\end{proposition}
\vspace{-0.5cm}
\begin{proof}
We know $\Umat = [\uvec_1, \ldots, \uvec_\xdim]$ for singular vectors $\uvec_i \in \RR^\xdim$ with $\uvec_i^\top \uvec_i = 1$ and $\uvec_i^\top \uvec_j = 0$ 
for $i\neq j$. Since $\Umat_k =  [\uvec_1, \ldots, \uvec_\rdim]$, we get that
$\Umat_k^\top \Umat = [\eye_{\rdim} \ \ \zerovec_{\xdim-\rdim}] \in \RR^{\rdim \times \xdim}$ for $\rdim$-dimensional identity matrix $\eye_k$ and zero matrix $\zerovec_{\xdim-\rdim} \in \RR^{\rdim \times (\xdim-\rdim)}$. Then we get that
$\Smat_L \bvec = \Umat_k^\top \bvec$ and 
%
$\Smat_L \Amat 
= [\eye_{\rdim} \ \ \zerovec_{\xdim-\rdim}] \Sigmamat \Vmat^\top
= \Sigmamat_\rdim \Vmat^\top
= \Sigmamat_\rdim \Vmat_\rdim^\top
.$
%
%
Therefore, $\wvec = (\Smat_L \Amat)^\pinv \Smat_L \bvec = \Vmat_\rdim \Sigmamat_\rdim^\pinv \Umat_\rdim^\top \bvec$.
\end{proof}

\section{LEFT-SIDED SKETCHING ALGORITHM}\label{sec_left}

In this section, we develop an efficient approach to use the smaller, sketched matrix $\Smat \Amat$
for incremental policy evaluation. The most straightforward way to use $\Smat\Amat$ is to 
incrementally compute $\Smat\Amat$, and periodically solve $\wvec = (\Smat\Amat)^\pinv \Smat\bvec$. This
costs $O(\xdim \rdim)$ per step, and $O(\xdim^2 \rdim)$ every time the solution is recomputed. To maintain
$O(\xdim \rdim)$ computation per-step, this full solution could only be computed every $\xdim$ steps,
which is too infrequent to provide a practical incremental policy evaluation approach. Further, because it is an underconstrained
system, there are likely to be infinitely many solutions to
$\Smat\Amat \wvec = \Smat\bvec$; amongst those solutions, we would like to sub-select amongst the unbiased solutions
to $\Amat \wvec = \bvec$. 

We first discuss how to efficiently maintain $(\Smat\Amat)^\pinv$, and then describe
how to use that matrix to obtain an unbiased algorithm. Let $\pAmat \defeq \Smat\Amat \in \RR^{\rdim \times \xdim}$.
For this underconstrained system with $\pbvec \defeq \Smat \bvec$, the minimum norm solution to $\pAmat \wvec = \pbvec$ is\footnote{We show in Proposition \ref{lem_rowrank}, Appendix \ref{app_theory}, that $\pAmat$ is full row rank with high probability. This property is required to ensure that the inverse of $\pAmat \pAmat^\top$ exists. In practice, this is less of a concern, because we initialize
the matrix $\pAmat_0 \pAmat_0^\top$ with a small positive value, ensuring invertibility for $\pAmat_t \pAmat_t^\top$ for finite $t$.}
$\wvec = \pAmat^\top (\pAmat \pAmat^\top)^\inv \pbvec$ and $\pAmat^\pinv = \pAmat^\top (\pAmat \pAmat^\top)^\inv \in \RR^{\xdim \times \rdim}$. 
To maintain $\pAmat_t^\pinv$ incrementally, therefore, we simply need to maintain $\pAmat_t$ incrementally 
and the $\rdim\times\rdim$-matrix $(\pAmat_t \pAmat_t^\top)^\inv$ incrementally. 
Let $\tilde{\evec}_t \defeq \Smat\evec_t$, $\dvec_t \defeq \dvec_t$
and $\hvec_t \defeq \pAmat_{t}\dvec_t$. We can update the sketched system in $O(\xdim \rdim)$ time and space
\begin{align*}
\pAmat_{\sampiter+1} &= \pAmat_{\sampiter} + \tfrac{1}{\sampiter+1}\left(\tilde{\evec}_{\sampiter} \dvec_{\sampiter}^\top -\pAmat_{\sampiter}\right)\\
\pbvec_{\sampiter+1} &= \pbvec_{\sampiter} + \tfrac{1}{\sampiter+1}\left(\tilde{\evec}_{\sampiter} R_{\sampiter+1} -\pbvec_{\sampiter}\right)
\end{align*}
To maintain $(\pAmat_t \pAmat_t^\top)^\inv$ incrementally, notice that the unnormalized update is
\begin{align*}
\pAmat_{t+1} \pAmat_{t+1}^\top 
&= (\pAmat_{t} + \tilde{\evec}_t \dvec_t^\top)(\pAmat_{t} +\tilde{\evec}_t \dvec_t^\top) \\
&= \pAmat_{t} \pAmat_{t}^\top + \tilde{\evec}_t  \hvec_t^\top + \hvec_t \tilde{\evec}_t ^\top + ||\dvec_t||_2^2 \| \tilde{\evec}_t  \tilde{\evec}_t ^\top
.
\end{align*}
Hence, $(\pAmat_{t+1} \pAmat_{t+1}^\top)^\inv$ can be updated from $(\pAmat_{t} \pAmat_{t}^\top)^\inv$ by applying
the Sherman-Morrison update three times. 
For a normalized update, based on samples, the update is
\begin{align*}
\pAmat_{t+1} \pAmat_{t+1}^\top 
&= \left(\tfrac{t}{t+1}\right)^2\pAmat_{t} \pAmat_{t}^\top 
+  \tfrac{t}{(t+1)^2}\left(\tilde{\evec}_{t}  \hvec_{t}^\top + \hvec_{t} \tilde{\evec}_{t} ^\top\right) \\
&+ \tfrac{1}{(t+1)^2}||\dvec_{t}||_2^2 \| \tilde{\evec}_{t}  \tilde{\evec}_{t}^\top
\end{align*}
%
%
We can then compute $\wvec_t = \pAmat_t (\pAmat_t \pAmat_t^\top)^\pinv \pbvec_t$ on each step.

This solution, however, will provide the minimum norm solution, rather than the unbiased solution,
even though the unbiased solution is feasible for the underconstrained system. To instead push the preference
towards this unbiased solution, we use the stochastic approximation algorithm, called ATD \citep{pan2017accelerated}.
This method is a quasi-second order method, that relies on a low-rank approximation $\hat{\Amat}_t$ to $\Amat_t$;
using this approximation, the update is
$\wvec_{t+1} = \wvec_t + (\stepsize_t\hat{\Amat}^\pinv_t + \eta \eye) \delta_t \evec_t$. 
Instead of being used to explicitly solve for $\wvec$, the approximation matrix is used to provide
curvature information. The inclusion of $\eta$ constitutes a small regularization component, that pushes
the solution towards the unbiased solution.  

We show in the next proposition that for our alternative approximation, we still obtain unbiased solutions. 
We use results for iterative methods for singular linear systems \citep{shi2011convergence,wang2013ontheconvergence}, since
$\Amat$ may be singular. $\Amat$ has been shown to be positive semi-definite under standard assumptions on the MDP \citep{yu2015onconvergence}; for simplicity, we assume $\Amat$ is positive semi-definite, instead of providing these MDP assumptions. 
\begin{assumption}
For $\Smat \in \RR^{\rdim \times \xdim}$ and
$\Bmat = \alpha (\Smat \Amat)^\pinv \Smat + \eta \eye$ with $\Bmat \in \RR^{\xdim \times \xdim}$, the matrix $\Bmat \Amat$
is diagonalizable. 
\end{assumption}
\begin{assumption}
$\Amat$ is positive semi-definite. 
\end{assumption}
\begin{assumption}
$\stepsize \in (0,\tfrac{1}{2})$ and $0 < \regwgt \le \tfrac{1}{2\eigmax(\Amat)}$ where $\eigmax(\Amat)$ is the maximum eigenvalue of $\Amat$.
\end{assumption}
\begin{theorem} \label{thm_main}
Under Assumptions 1-3, 
the expected updating rule 
$\wvec_{t+1} = \wvec_t + \E_\pi[\Bmat \delta_t \evec_t]$
 converges to a fixed-point $\wvec^\star = \Amat^\pinv \bvec$. 
\end{theorem}
\begin{proof} 
The expected updating rule is $\E_\pi[\Bmat \delta_t \evec_t] = \Bmat(\bvec - \Amat \wvec_t)$. 
As in the proof of convergence for ATD \cite[Theorem 1]{pan2017accelerated}, 
we similarly verify the conditions from \citep[Theorem 1.1]{shi2011convergence}. 

Notice first that \ \ \ \
$\Bmat \Amat = \stepsize (\Smat \Amat)^\pinv \Smat \Amat + \regwgt \Amat $.
%

For singular value decomposition, $\Smat \Amat = \Umat \Sigmamat \Vmat^\top$, we have that 
$(\Smat \Amat)^\pinv \Smat \Amat = \Vmat \Sigmamat^\pinv \Umat^\top \Umat \Sigmamat \Vmat^\top
= \Vmat [\eye_{\tilde{\rdim}} \ \zerovec_{\xdim -\rdim}] \Vmat^\top$, where $\tilde{\rdim} \le \rdim$ is the rank of $\Smat\Amat$. The maximum eigenvalue of 
$(\Smat \Amat)^\pinv \Smat \Amat$ is therefore $1$. 

Because $(\Smat \Amat)^\pinv \Smat \Amat $ and
$\Amat$ are both positive semidefinite, $\Bmat \Amat$ is positive semi-definite. 
By Weyl's inequalities, 
\begin{equation*}
\eigmax(\Bmat \Amat) \le \stepsize\eigmax( (\Smat \Amat)^\pinv \Smat \Amat) + \regwgt \eigmax(\Amat)
.
\end{equation*}
Therefore, the eigenvalues of $\eye - \Bmat \Amat$ have absolute value strictly less than 1,
because $\regwgt \le (2\eigmax(\Amat))^\inv$ and $\stepsize < 1/2 = (2\eigmax( (\Smat \Amat)^\pinv \Smat \Amat))^\inv $
by assumption. 

For the second condition, since $\Bmat\Amat$ is PSD and diagonalizable, we
can write $\Bmat\Amat = \Qmat \Lambdamat \Qmat^\inv$ for some matrices $\Qmat$ and diagonal matrix $\Lambdamat$
with eigenvalues greater than or equal to zero. Then $(\Bmat\Amat)^2 = \Qmat \Lambdamat \Qmat^\inv\Qmat \Lambdamat \Qmat^\inv
= \Qmat \Lambdamat^2 \Qmat^\inv$ has the same rank. 


\newcommand{\nullspace}{\text{nullspace}}
For the third condition, because $\Bmat \Amat$ is the sum of two positive semi-definite matrices,
the nullspace of $\Bmat\Amat$ is a subset of the nullspace of each of those matrices individually:
$\nullspace(\Bmat\Amat) = \nullspace(\stepsize (\Smat \Amat)^\pinv \Smat \Amat + \regwgt \Amat ) \subseteq \regwgt \nullspace( \regwgt \Amat ) = \nullspace(\Amat)$. In the other direction, for all $\wvec$ such that $\Amat\wvec = \zerovec$, its clear
that $\Bmat \Amat \wvec = \zerovec$, and so $\nullspace(\Amat) \subseteq \nullspace(\Bmat\Amat)$.
Therefore, $\nullspace(\Amat) = \nullspace(\Bmat\Amat)$.
%
\end{proof}

\begin{figure*}[t!]
	\vspace{-0.5cm}
	\subfigure[Mountain Car, RBF]{
		\includegraphics[width=\figwidththree]{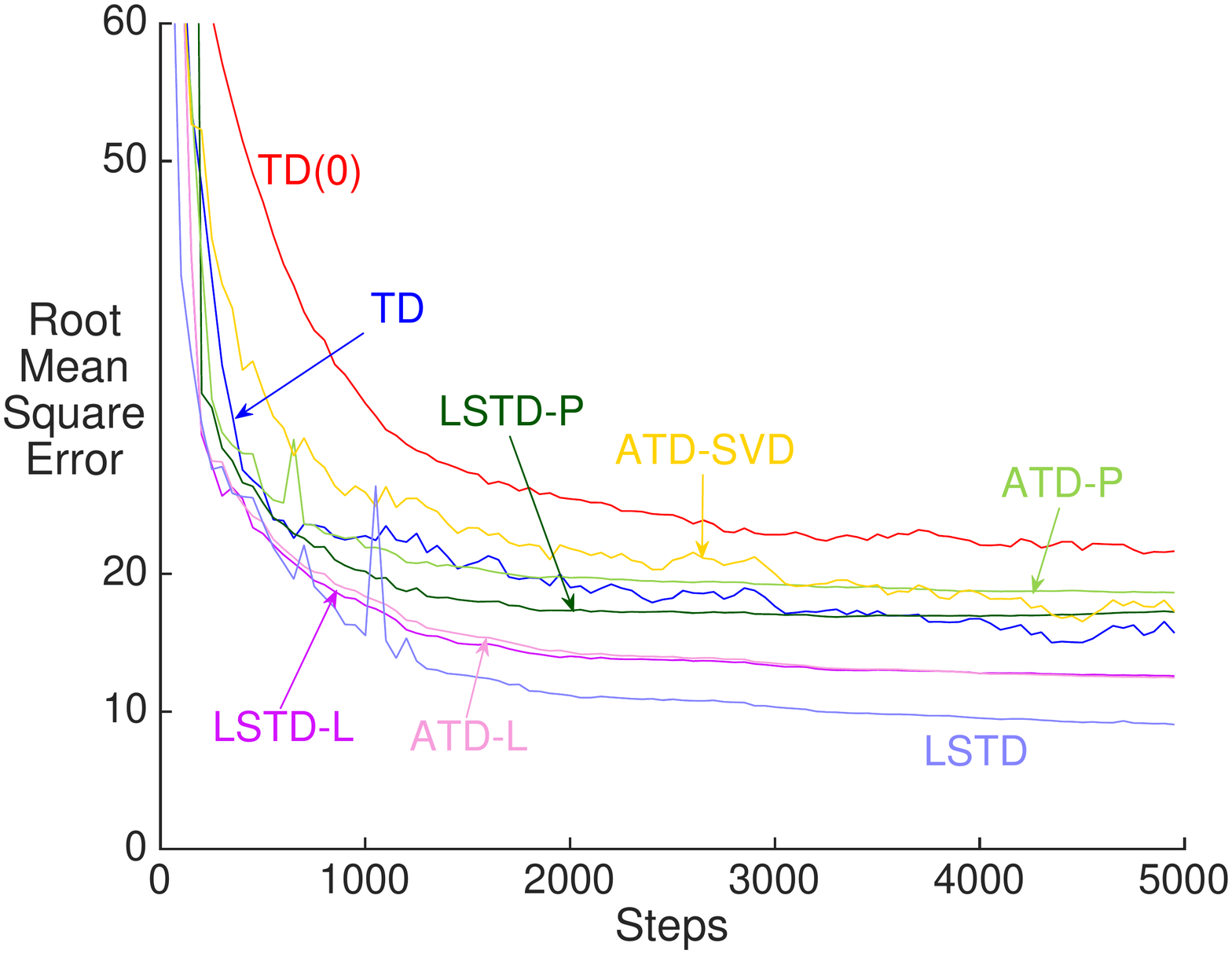} \label{fig:mcar_gau_rbf50r}} 
	\subfigure[Mountain Car, Tile coding]{
		\includegraphics[width=\figwidththree]{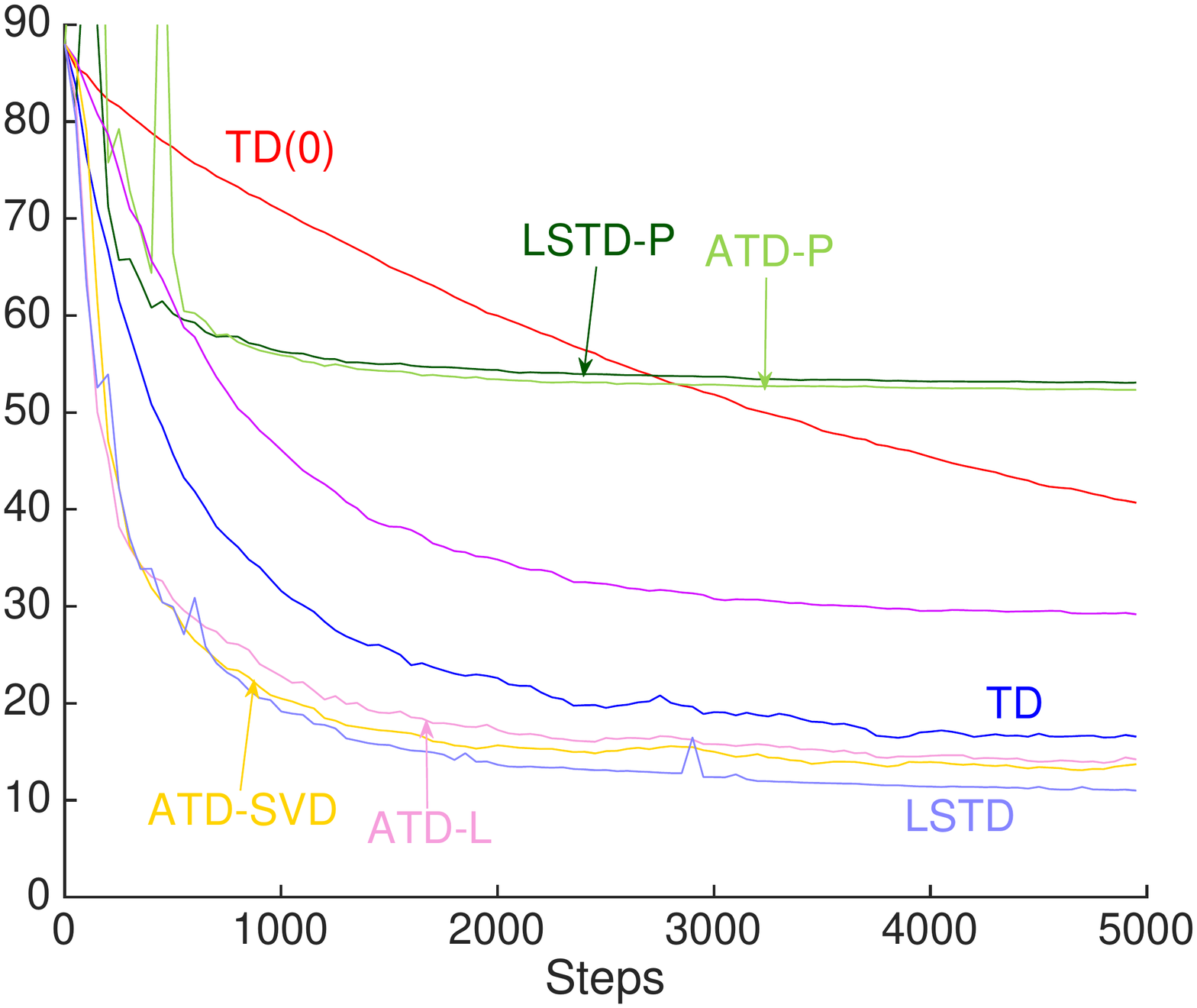} \label{fig:mcar_gau_tile50r}}\\
	\subfigure[Mountain Car, RBF, Sensitivity]{
		\includegraphics[width=\figwidththree]{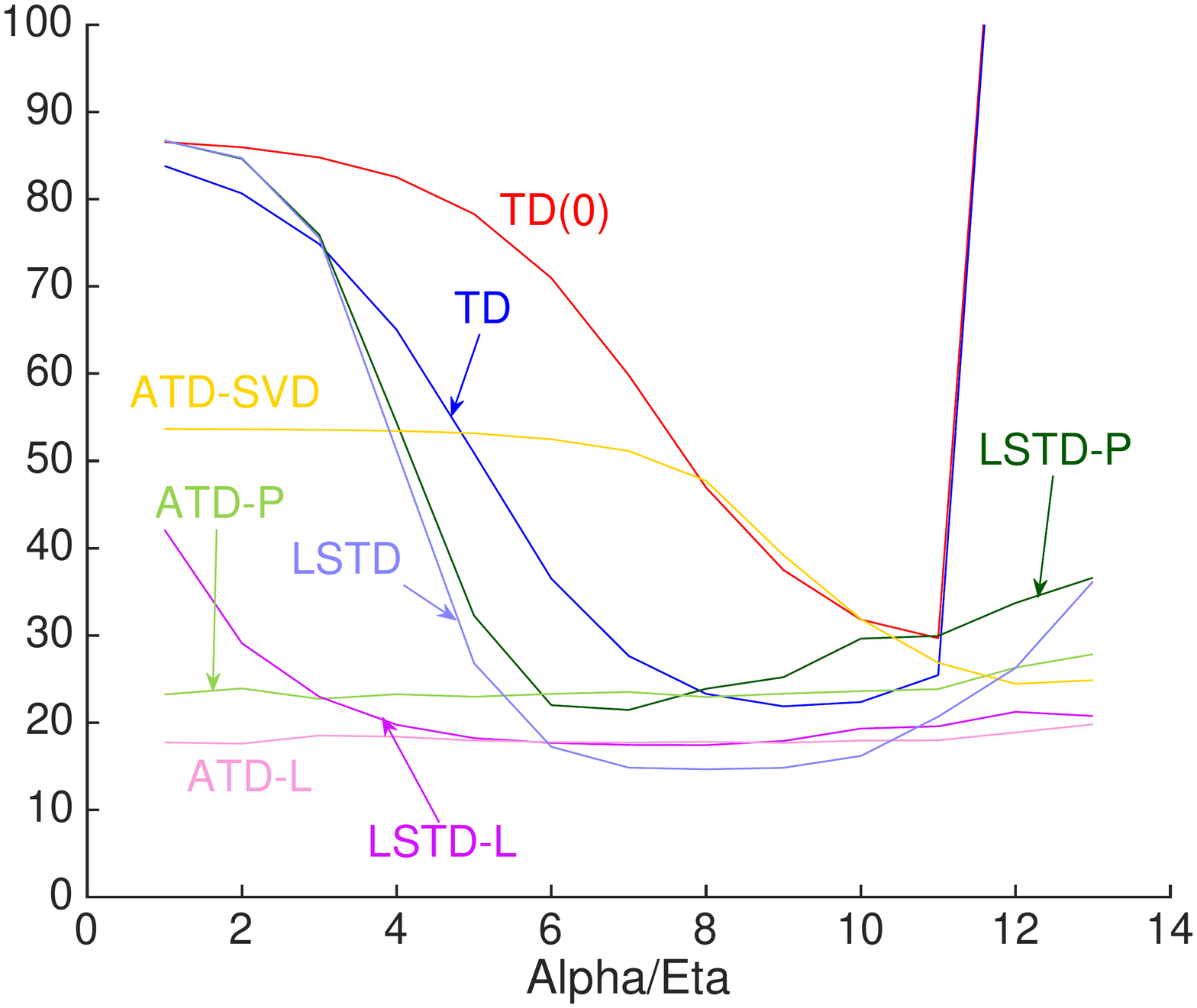}\label{fig:mcar_gau_rbf50rsensi}}
	\subfigure[Mountain Car, Tile coding, Sensitivity]{
		\includegraphics[width=\figwidththree]{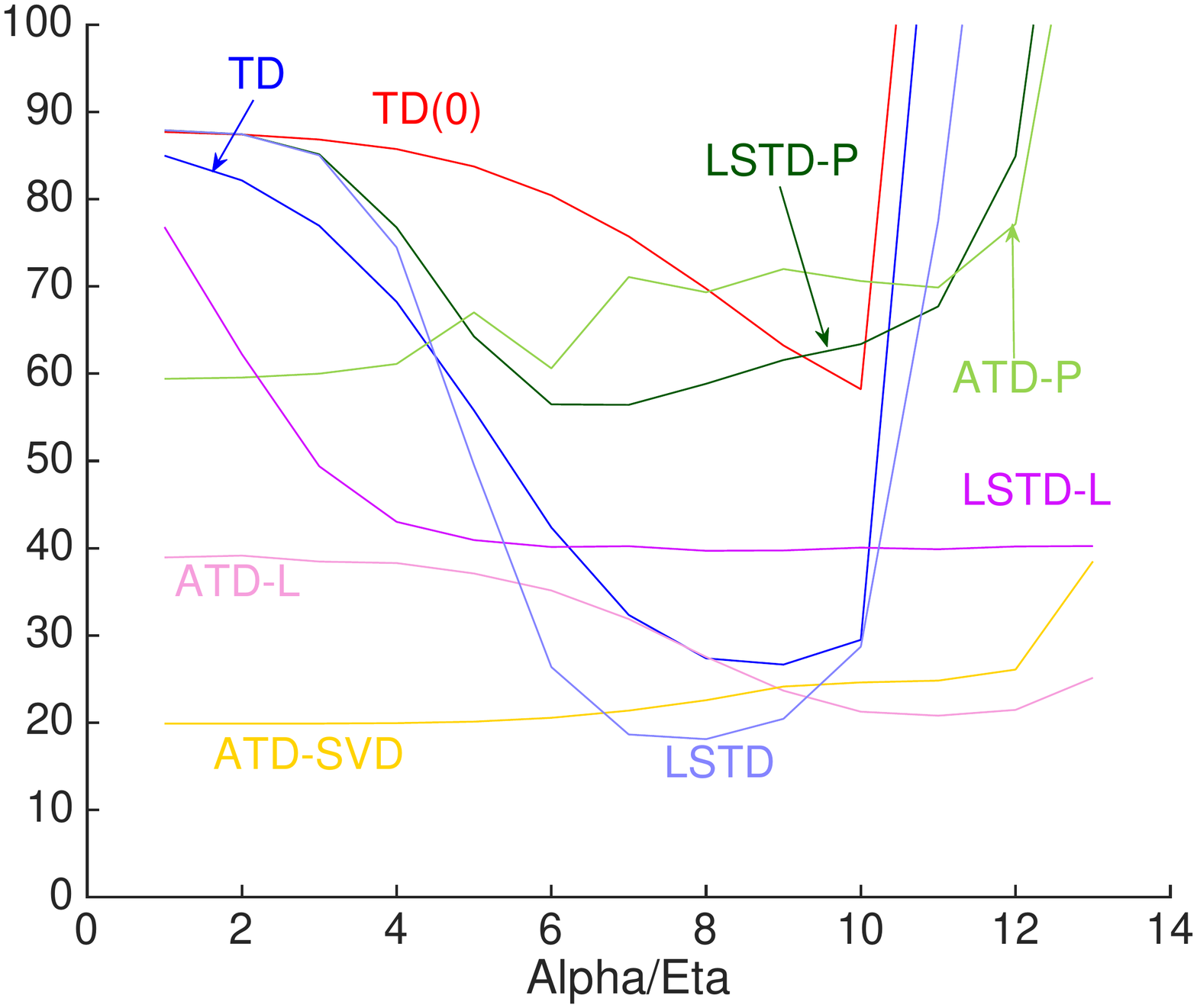}\label{fig:mcar_gau_tile50rsensi}}	
	\begin{minipage}{\figwidththree}
	\vspace{-8.0cm}
	\caption{ 
		 \textbf{(a)} and \textbf{(b)} are learning curves on Mountain Car with $\rdim = 50$, and \textbf{(c)} and \textbf{(d)} are their corresponding parameter-sensitivity plots. The sensitivity plots report average RMSE over the entire learning curve, for the best $\lambda$ for each parameter. The stepsize $\stepsize$ is reported for TD, the initialization parameter $\xi$ for the LSTD methods
		and the regularization parameter $\eta$ for the ATD methods. The initialization for the matrices in the ATD methods is fixed to the identity.  The range for the regularization term $\eta$ is $0.1$ times the range for $\stepsize$. As before, the sketching approaches with RBFs perform better than with tile coding. The sensitivity of the left-side projection methods is significantly lower than the TD methods. ATD-L also seems to be less sensitive than ATD-SVD, and incurs less bias than LSTD-L. 
	}\label{fig:mcar_lc}
\end{minipage}
\end{figure*}

\section{\!\!WHEN SHOULD SKETCHING HELP?}

To investigate the properties of these sketching approaches, we need to understand
when we expect sketching to have the most benefit. 
Despite the wealth of literature on sketching and strong theoretical results, there seems to be fewer empirical investigations
into when sketching has most benefit. In this section, we elucidate some hypotheses
about when sketching should be most effective, which we then explore in our experiments. 

In the experiments for sketching the features in Section \ref{sec_issues},
it was clear that sketching the RBF features was much more effective than sketching the tile coding features. 
A natural investigation, therefore, is into the properties of representations that
are more amenable to sketching. The key differences between these two representations is in 
terms of smoothness, density and overlap. The tile coding representation has non-smooth 0,1 features,
which do not overlap in each grid. Rather, the overlap for tile coding results from
overlapping tilings.
This differs from RBF overlap, where centers are arranged in a grid and only edges of the RBF features overlap. The density of RBF features is significantly higher,
since more RBFs are active for each input. Theoretical work in sketching for regression \citep{maillard2012linear}, however, does not require features to be smooth. 
We empirically investigate these three properties---smoothness, density and overlap. 

There are also some theoretical results that suggest sketching could be more amenable for more distinct features---less overlap or potentially less tilings.
\citet{balcan2006kernels} showed a worst-case setting where data-independent sketching results in poor performance.
They propose a two-stage projection, to maintain separability in classification. 
The first stage uses a data-dependent projection, to ensure features are not highly correlated, and the second uses a data-independent projection (a sketch)
to further reduce the dimensionality after the orthogonal projection. The implied conclusion from this result
is that, if the features are not highly correlated, then the first step can be avoided and the data independent sketch 
should similarly maintain classification accuracy. 
This result suggests that sketching for feature expansions with less redundancy should
perform better. 

We might also expect sketching to be more robust to the condition number of the matrix. 
For sketching in regression, \cite{fard2012compressed} found a bias-variance trade-off when increasing
$\rdim$, where for large $\rdim$, estimation error from a larger number of parameters became
a factor. Similarly, in our experiments above, LSTD using an incremental Sherman-Morrison update
has periodic spikes in the learning curve, indicating some instability. 
The smallest eigenvalue of the
sketched matrix should be larger than that of the original matrix;
this improvement in condition number compensates for the loss in information.
Similarly, we might expect that maintaining an incremental singular value decomposition,
for ATD, could be less robust than ATD with left-side sketching.

\begin{figure*}[htp!]
	\centering
	\subfigure[Puddle World, RBF, $\rdim = 25$]{
		\includegraphics[width=\figwidththree]{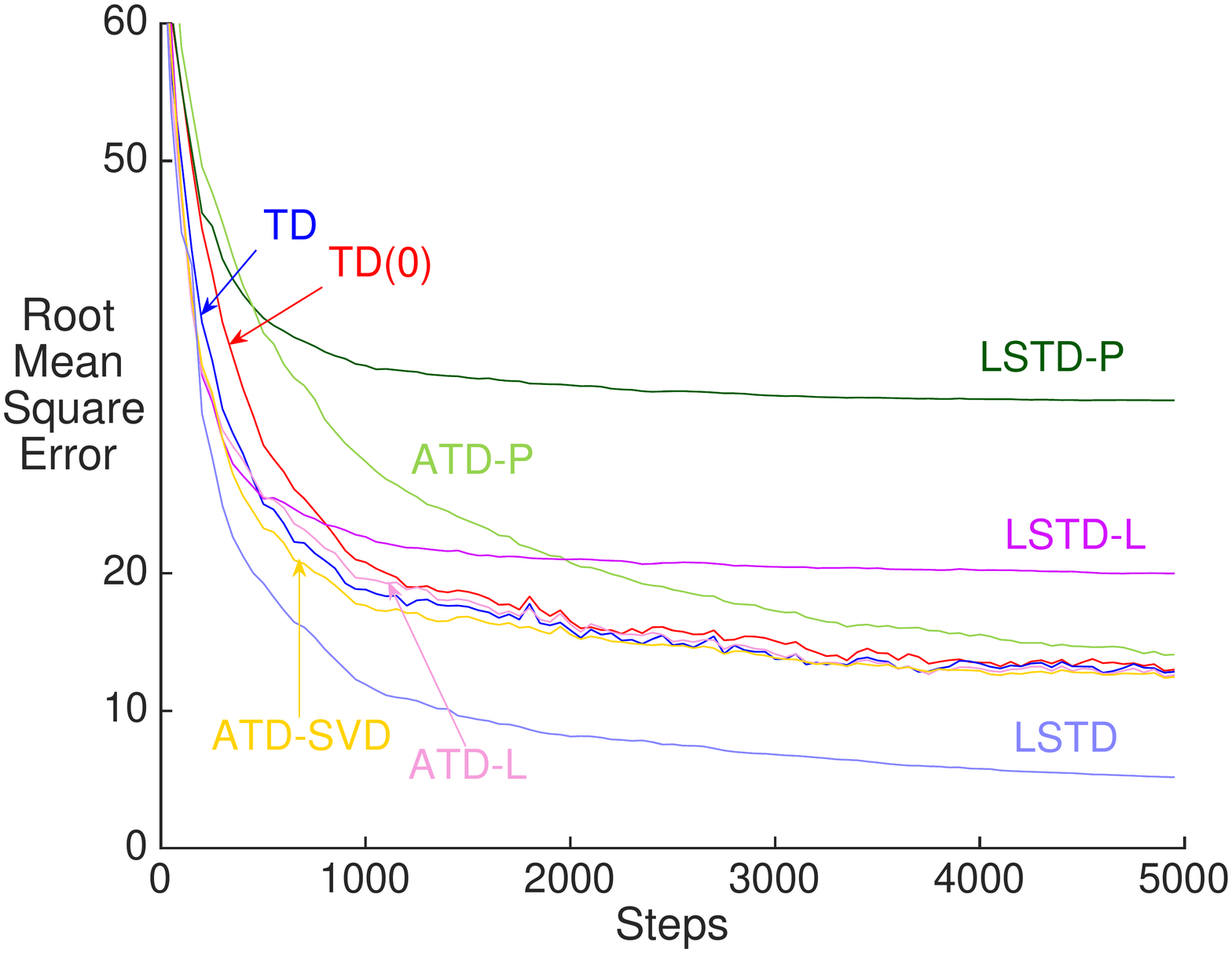} \label{fig:pd_gau_rbf25r}} 
	\subfigure[Puddle World, RBF, $\rdim = 50$]{
		\includegraphics[width=\figwidththree]{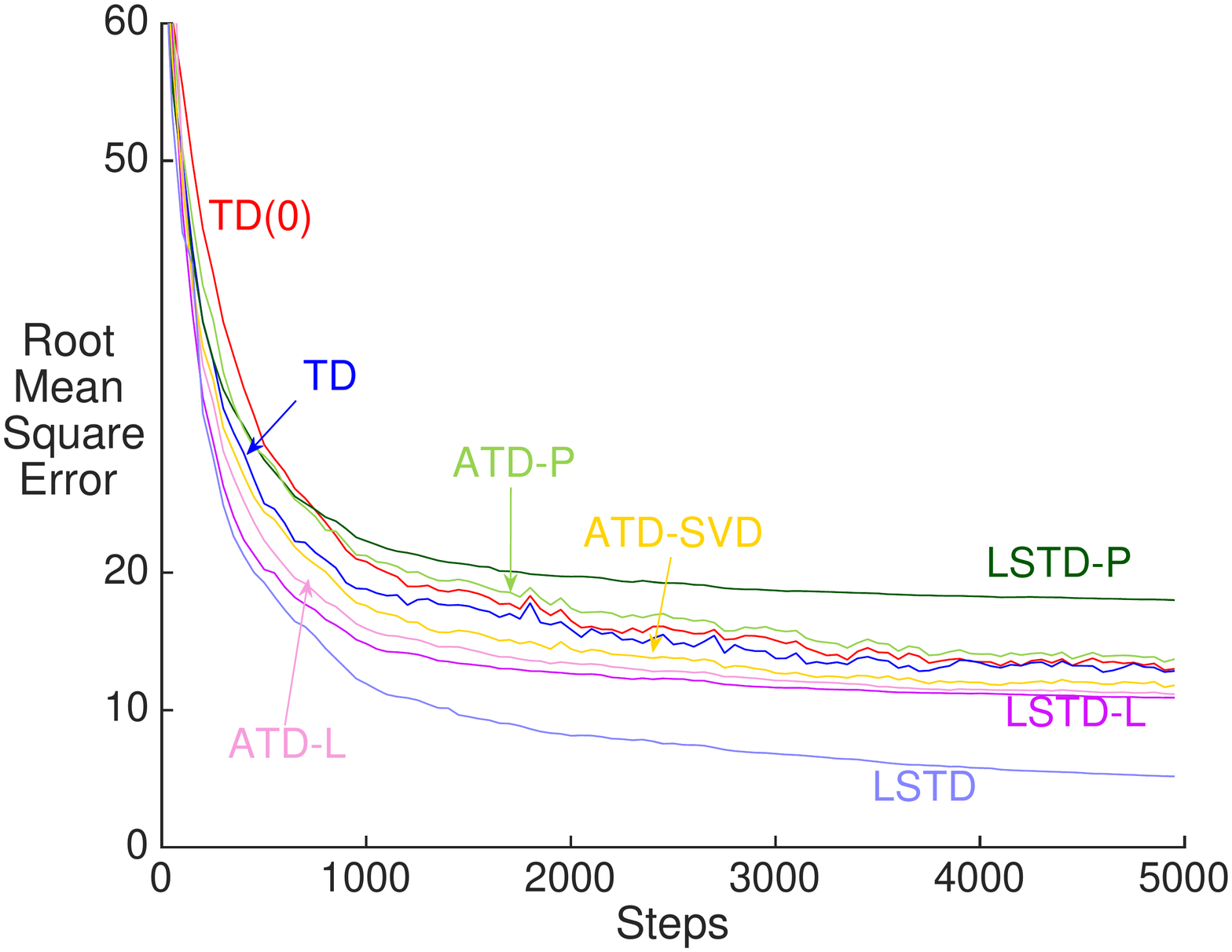} \label{fig:pd_gau_rbf50r}}
	\subfigure[Puddle World, RBF, $\rdim = 75$]{
		\includegraphics[width=\figwidththree]{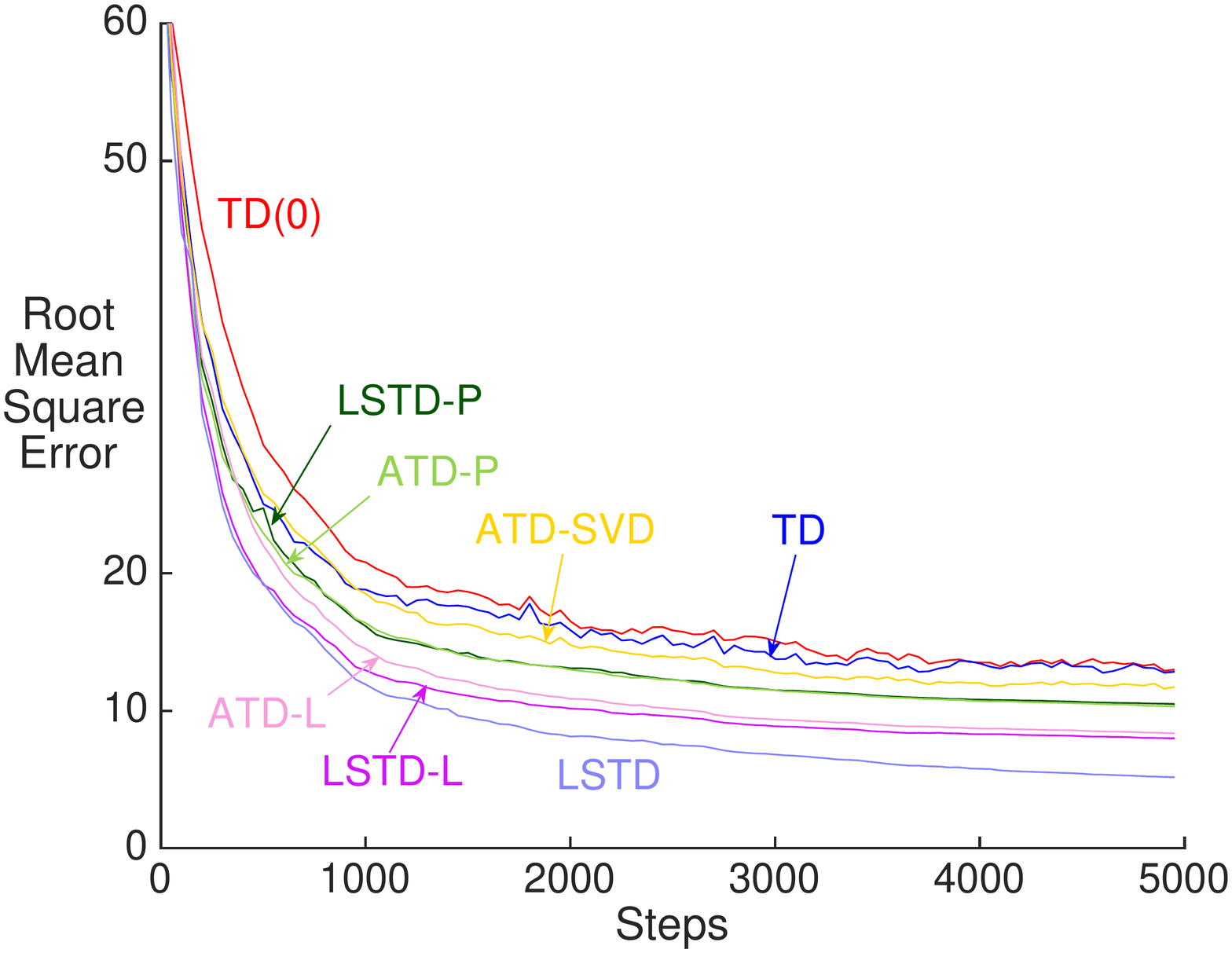}\label{fig:pd_gau_rbf75r}}
		\caption{ 
			Change in performance when increasing $\rdim$, from $25$ to $75$.  Two-sided projection (i.e., projecting the features) significantly improves with larger $\rdim$, but is strictly dominated by left-side projection. At $\rdim = 50$, the left-side projection methods are outperforming TD
			and are less variant. ATD-SVD seems to gain less with increasing $\rdim$, though in general we found ATD-SVD to perform more poorly than ATD-P particularly for RBF representations.
		}\label{fig:pd_rank_lc}
\end{figure*}

\begin{figure*}[htp!]
\vspace{-0.5cm}
	\subfigure[RBF, $\rdim = 50$]{
		\includegraphics[width=\figwidthfour]{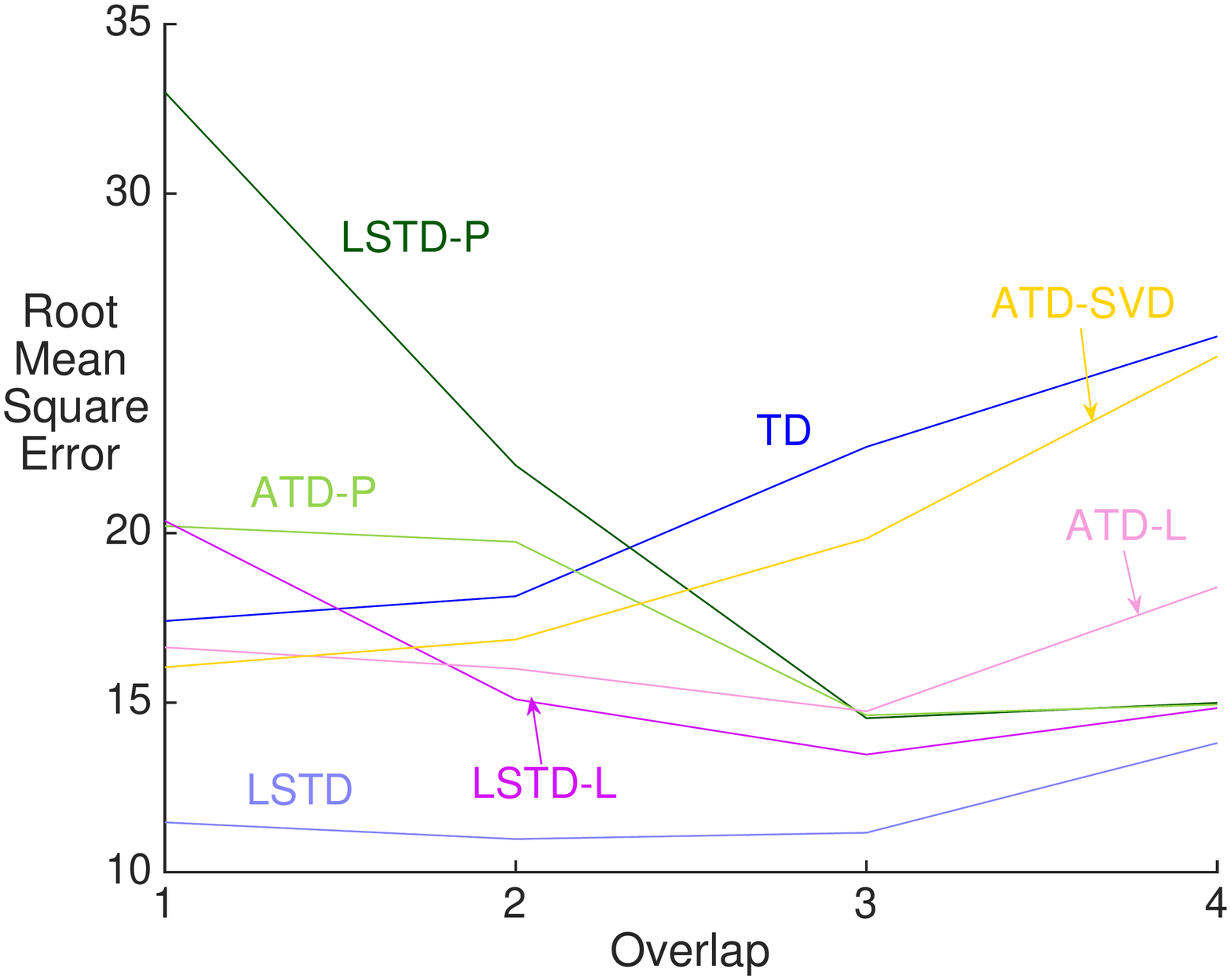} \label{fig:pd_dense_tile50r}}
	\subfigure[Spline, $\rdim = 50$]{
		\includegraphics[width=\figwidthfour]{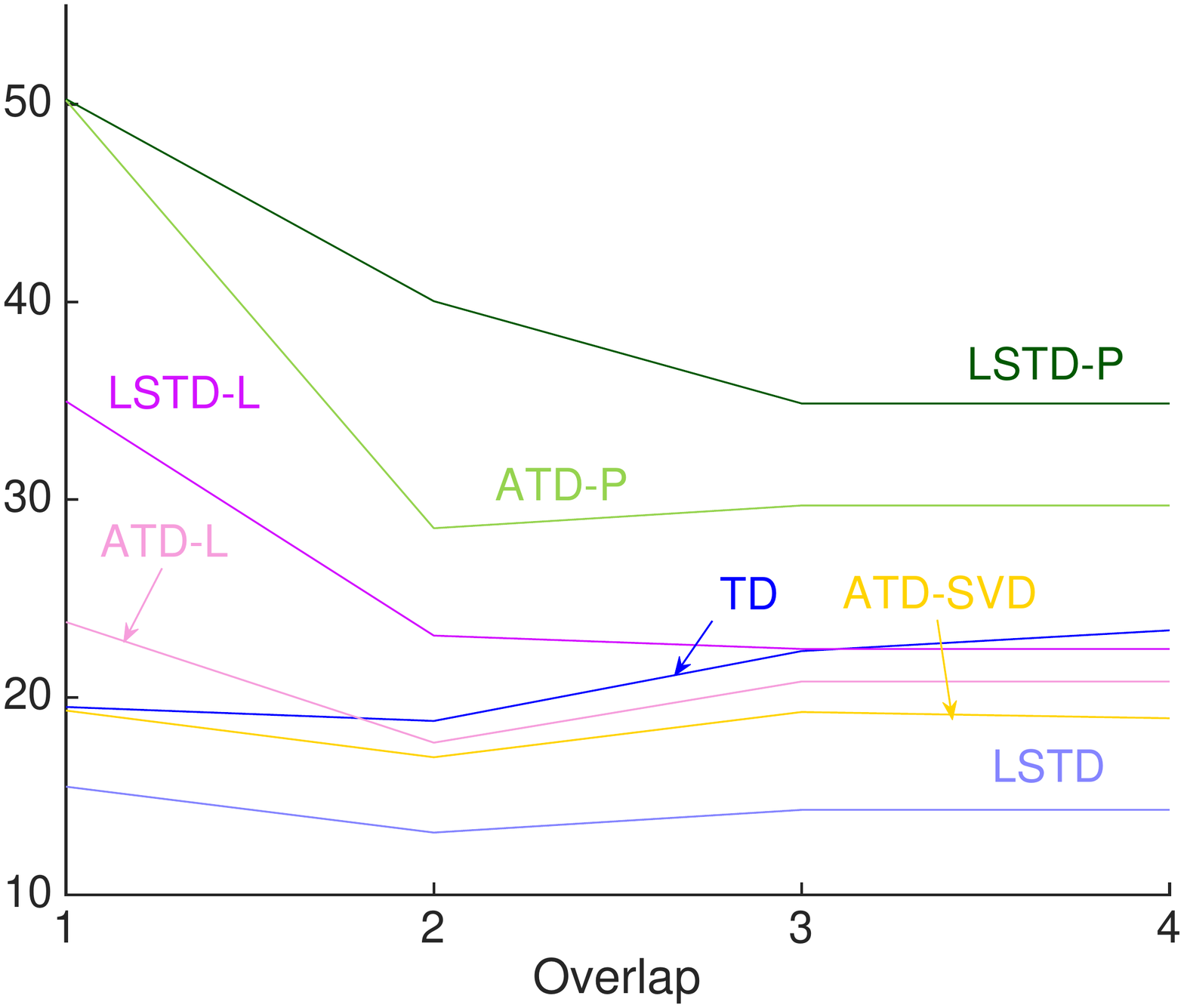} \label{fig:pd_dense_rbf50r}} 
	\subfigure[Tile coding, $\rdim = 50$]{
		\includegraphics[width=\figwidthfour]{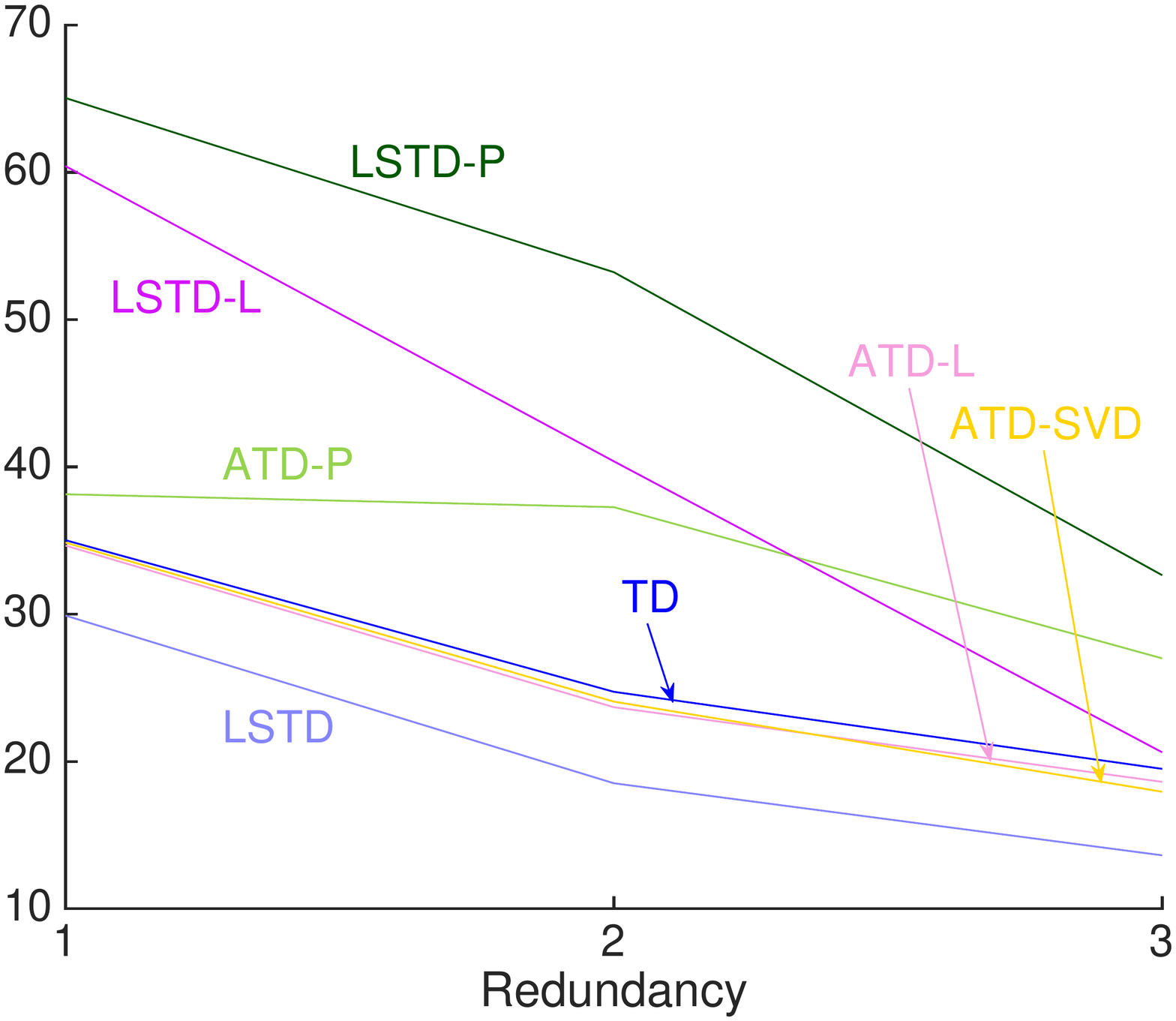} \label{fig:pd_tile_redund}}		
	\subfigure[RBFs with tilings, $\rdim = 50$]{
			\includegraphics[width=\figwidthfour]{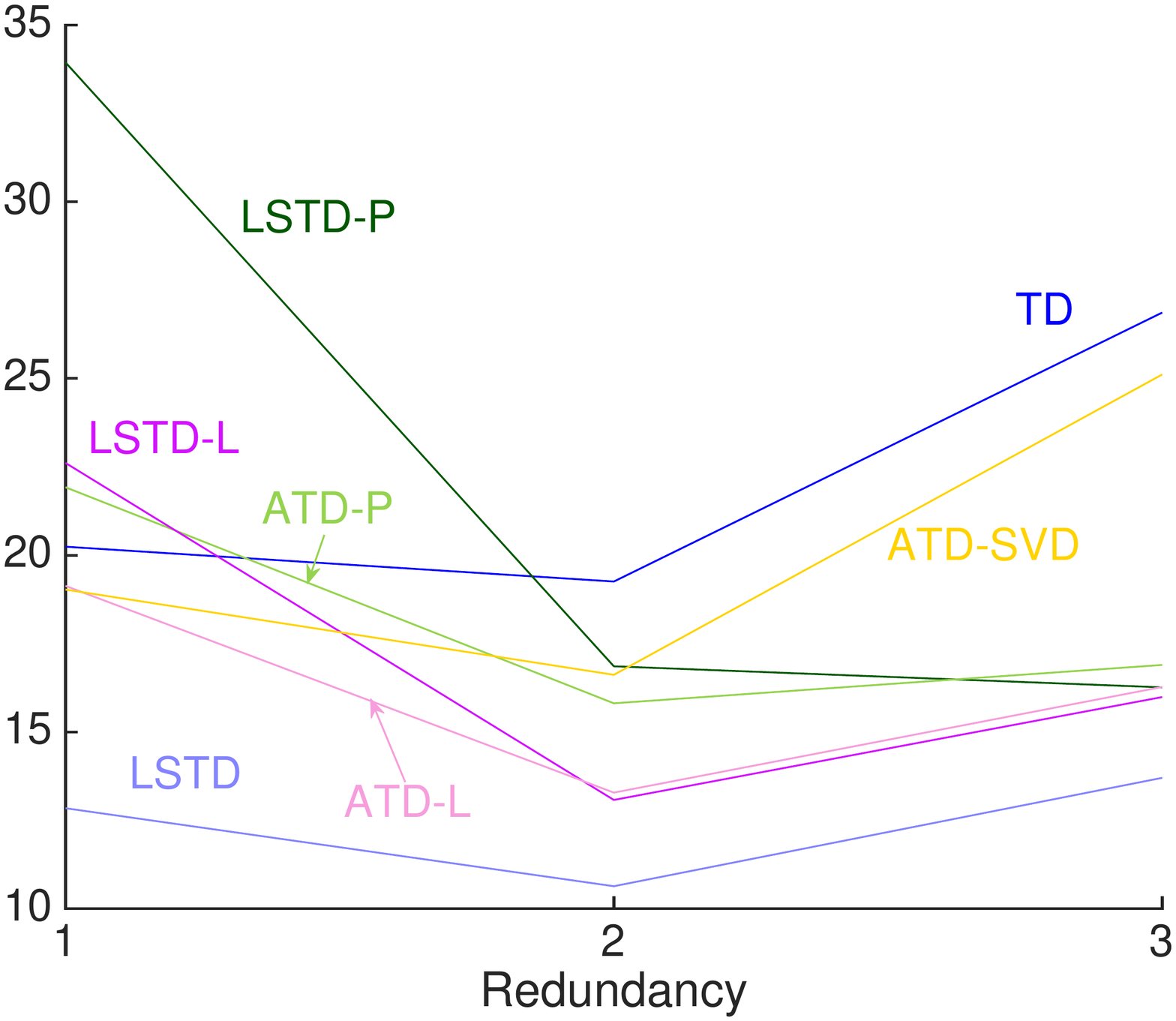} \label{fig:pd_tilerbf_redund}}
	\caption{ 
		The effect of varying the representation properties, in Puddle World with $\xdim = 1024$. In \textbf{(a)} and \textbf{(b)}, we examine the impact of varying the overlap, for both smooth features (RBFs) and 0-1 features (Spline). For spline, the feature is 1 if $||\xvec - \cvec_i|| < \sigma$ and otherwise 0. The spline feature represents a bin, like for tile coding, but here we adjust the widths of the bins so that they can overlap and do not use tilings. 
		The x-axis has four width values, to give a corresponding feature vector norm of about $20, 40, 80, 120$. 
		In \textbf{(c)} and \textbf{(d)}, we vary the redundancy, where number of tilings is increased and the total number of features kept constant. We generate tilings for RBFs like for tile coding, but for each grid cell use an RBF similarity rather than a spline similarity. We used $4 \times 16 \times 16, 16 \times 8 \times 8$ and  $64 \times 4 \times 4$. 
	}\label{fig:pd_overlap_redund}
\end{figure*}

\section{EXPERIMENTS}

In this section, we test the efficacy of sketching for LSTD and ATD in four domains: Mountain Car, Puddle World, Acrobot and Energy Allocation. We  set $\rdim = 50$, unless otherwise specified, average all results over 50 runs and sweep parameters for each algorithm. Detailed
experimental settings, such as parameter ranges, are in Appendix \ref{app_experiments}. To distinguish projections, we add -P for two-sided and -L for left-sided to the algorithm name.

We conclude that 1) two-sided projection---projecting the features---generally does much worse than only projecting the left-side of $\Amat$, 2) higher feature density is more amenable to sketching, particularly for two-sided sketching, 3) smoothness of features only seems to impact two-sided sketching, 4) ATD with sketching decreases bias relative to its LSTD variant and 5) ATD with left-sided sketching typically performs as well as ATD-SVD, but is significantly faster. 


\textbf{Performance and parameter sensitivity for RBFs and Tile coding.} We first more exhaustively compare the algorithms in Mountain Car and Puddle World, in Figures \ref{fig:mcar_lc} and \ref{fig:pd_rank_lc} with additional such results in the appendix. 
As has been previously observed, TD with a well-chosen stepsize can perform almost as well as LSTD in terms of sample efficiency,
but is quite sensitive to the stepsize. Here, therefore, we explore if our matrix-based learning algorithms can reduce this parameter sensitivity. In Figure \ref{fig:mcar_lc}, we can indeed see that this is the case. The LSTD algorithms look a bit more sensitive, because we sweep over small initialization values for completeness. For tile coding, the range is a bit more narrow, but for RBFs, in the slightly larger range, the LSTD algorithms are quite insensitive for RBFs. Interestingly, LSTD-L seems to be more robust.
We hypothesize that the reason for this is that LSTD-L only has to initialize a smaller $\rdim\times\rdim$ symmetric matrix, $(\Smat \Amat (\Smat\Amat)^\top)^\inv = \eta \eye$, and so is much more robust to this initialization. In fact, across settings, we found initializing to $\eye$ was effective. Similarly, ATD-L benefits from this robustness, since it needs to initialize the same matrix, and then further overcomes bias using the approximation to $\Amat$ only for curvature information. 


\textbf{Impact of the feature properties.} 
We explored the feature properties---smoothness, density, overlap and redundancy---where we hypothesized sketching should help, shown in Figure~\ref{fig:pd_overlap_redund}. 
The general conclusions are 1) the two-side sketching methods improve---relative to LSTD---with increasing density (i.e., increasing overlap and increasing redundancy), 2) the smoothness of the features (RBF versus spline) seems to affect the two-side projection methods much more, 3) the shape of the left-side projection methods follows that of LSTD and 4) ATD-SVD appears to follow the shape of TD more. Increased density generally seemed to degrade TD, and so ATD-SVD similarly suffered more in these settings. In general, the ATD methods had less gain over their corresponding LSTD variants, with increasing density. 

\textbf{Experiments on high dimensional domains.} We finally apply our sketching techniques on two high dimensional domains to illustrate practical usability: Acrobot and Energy allocation. The Acrobot domain \citep{sutton1998reinforcement} is a four dimensional episodic task, where the goal is to raise an arm to a certain height. 
The Energy allocation domain \cite{salas2013benchmarking} is a five-dimensional continuing task, where the goal is to store and allocate energy to maximize profit. For Acrobot, we used $14,400$ uniformly-spaced centers and for Energy allocation, we used the same tile coding of $8192$ features as \citet{pan2017accelerated}. We summarize the results in the caption of Figure \ref{fig:high_domain}, with the overall conclusion that ATD-L provides an attractive way to reduce parameter sensitivity of TD, and benefit from sketching to reduce computation. 

\begin{figure*}[htp!]
\vspace{-0.5cm}
	\centering
	\subfigure[Acrobot, RBF]{
		\includegraphics[width=\figwidththree]{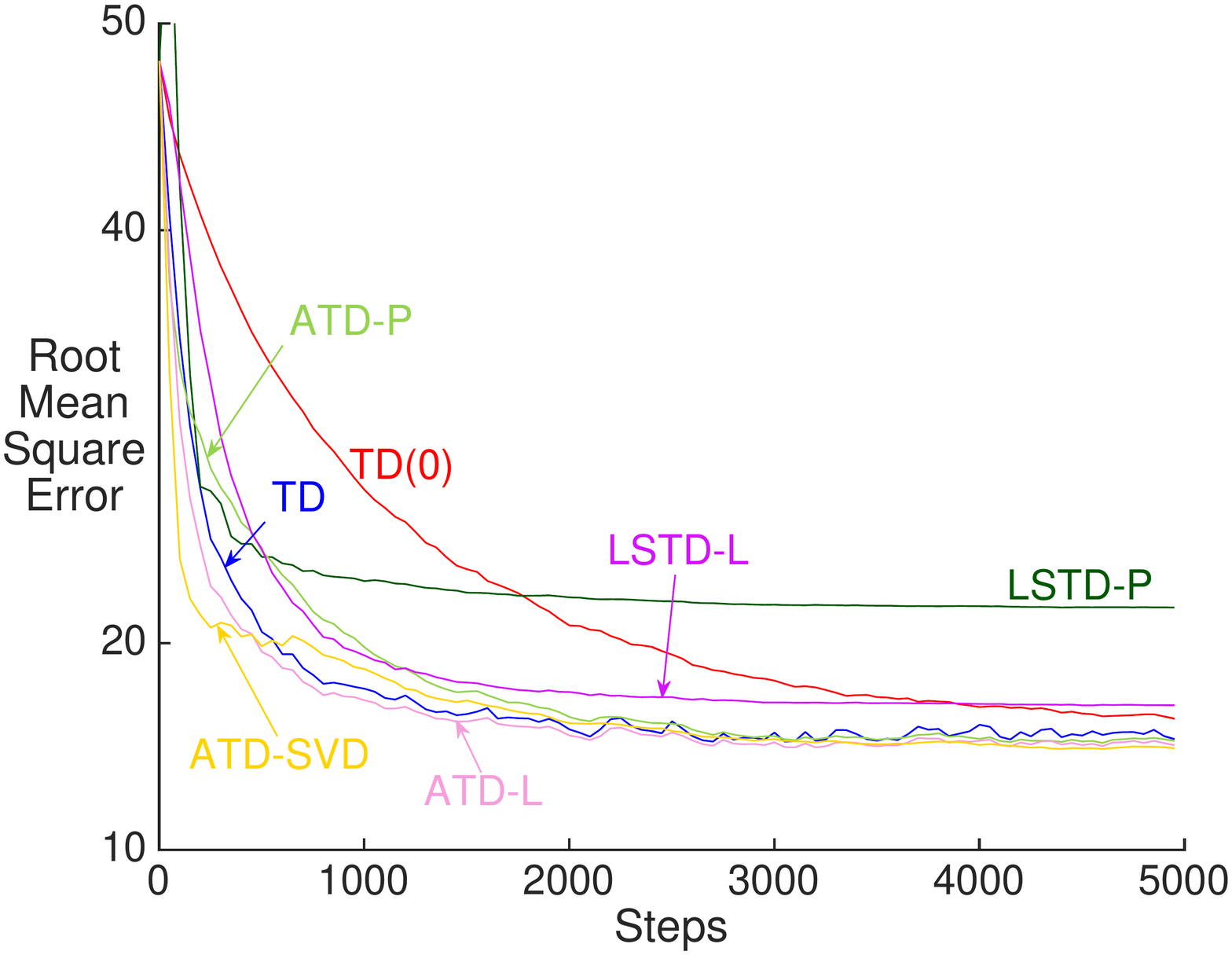}\label{fig:acro_rbf50r}}	
	\subfigure[Acrobot, RMSE vs Time]{
		\includegraphics[width=\figwidththree]{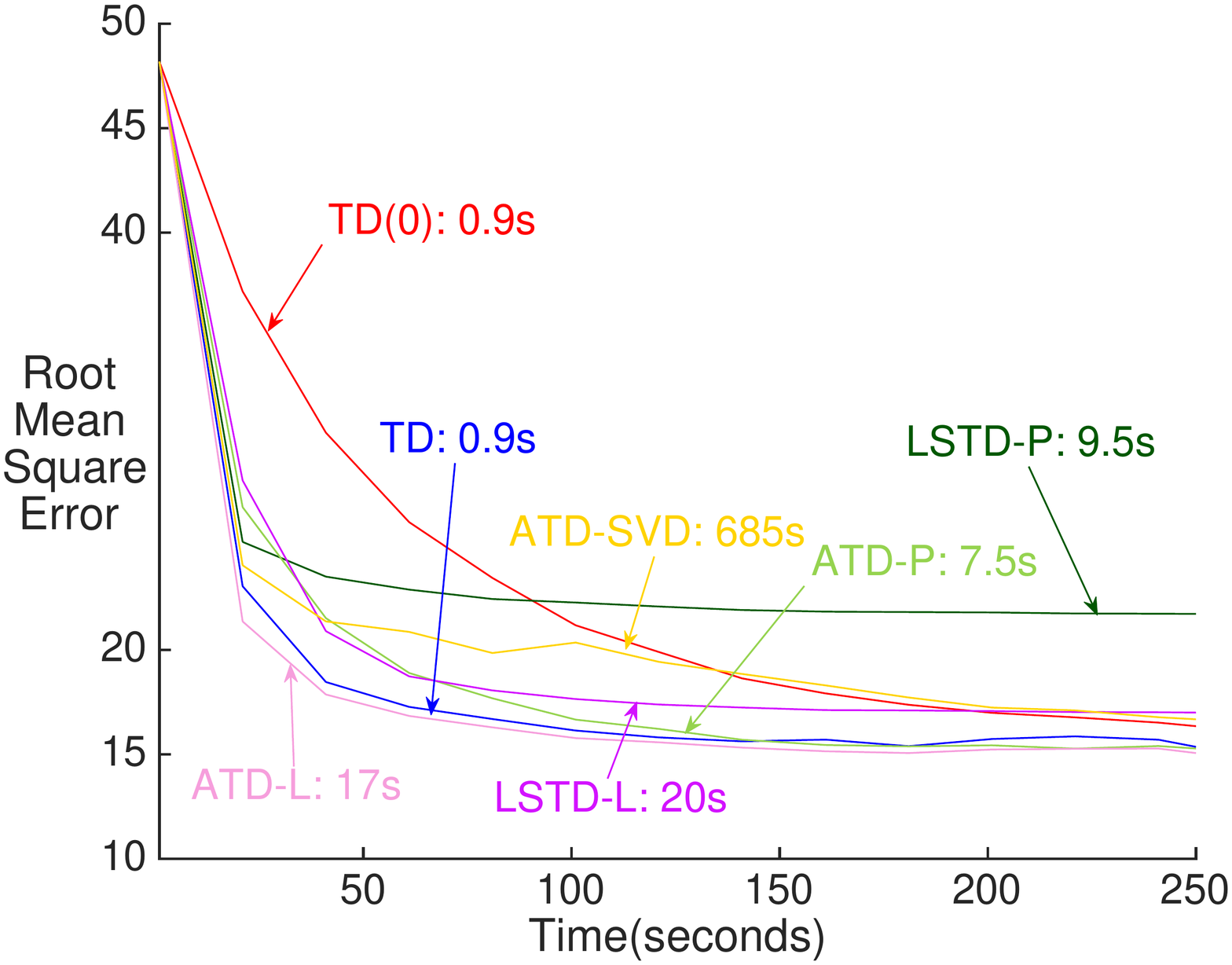}}
	\subfigure[Energy allocation, Tile coding]{
		\includegraphics[width=\figwidththree]{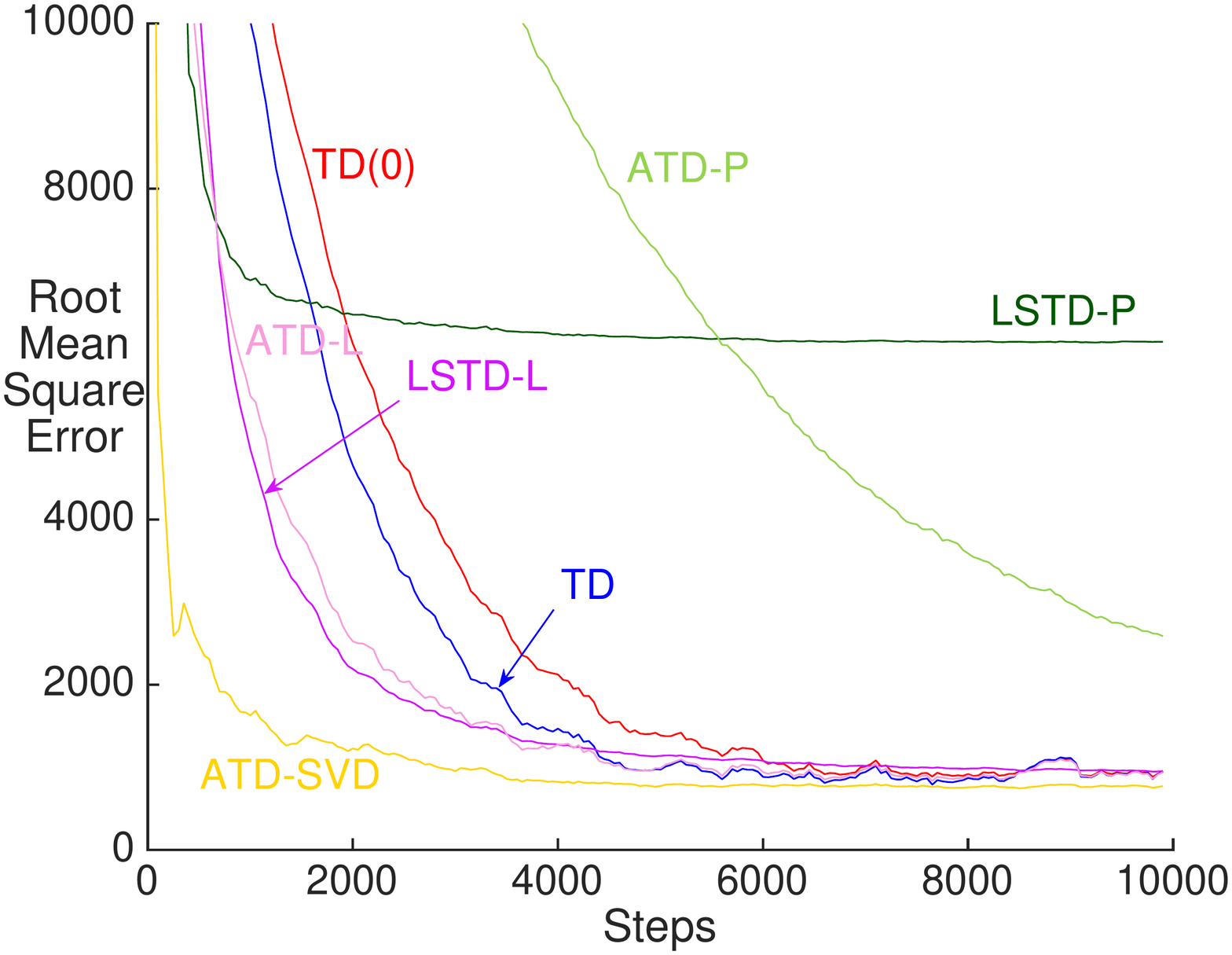}\label{fig:energy_tile50r}}
	\caption{ 
		Results in domains with high-dimensional features, using $\rdim = 50$ and with results averaged over $30$ runs. 
		For Acrobot, the (left-side) sketching methods perform well and are much less sensitive to parameters than TD. For runtime, we show RMSE versus time allowing the algorithms to process up to $25$ samples per second, to simulate a real-time setting learning; slow algorithms cannot process all $25$ within a second. With computation taken into account, ATD-L has a bigger win over ATD-SVD, and does not lose relative to TD. Total runtime in seconds for one run for each algorithm is labeled in the plot. ATD-SVD is much slower, because of the incremental SVD. 
		For the Energy Allocation domain, the two-side projection methods (LSTD-P, ATD-P) are significantly worse than other algorithms. Interestingly, here ATD-SVD has a bigger advantage, likely because sketching the tile coding features is less effective. 
	}\label{fig:high_domain}
\end{figure*}

\section{CONCLUSION AND DISCUSSION}

In this work, we investigated how to benefit from sketching approaches
for incremental policy evaluation. We first showed that sketching features can have significant bias issues,
and proposed to instead sketch the linear system, enabling better control over how much information is lost. 
We highlighted that sketching for radial basis features seems to be much more effective, than for tile coding,
and further that a variety of natural asymmetric sketching approaches for sketching the linear system are not effective.
We then showed that more carefully using sketching---particularly with left-side sketching within a quasi-Newton update---enables us to obtain an unbiased approach that can improve sample efficiency without incurring significant computation.  
Our goal in this work was to provide practical methods that can benefit from sketching, 
and start a focus on empirically investigating settings in which sketching is effective. 

Sketching has been used for quasi-Newton updates in online learning; a natural
question is if those methods are applicable for policy evaluation.
\citet{luo2016efficient} consider sketching approaches for an online Newton-update, for general functions
rather than just the linear function approximation case we consider here.
They similarly have to consider updates amenable to incrementally approximating a matrix (a Hessian in their case).
In general, however, porting these quasi-Newton updates to policy evaluation for reinforcement learning is problematic
for two reasons. First, the objective function for temporal difference learning is the mean-squared projected Bellman error,
which is the product of three expectations. It is not straightforward to obtain an unbiased sample of this gradient, which is why \citet{pan2017accelerated} propose a slightly different quasi-Newton update that uses $\Amat$ as a preconditioner. Consequently, it is
not straightforward to apply quasi-Newton online algorithms that assume access to unbiased gradients. 
Second, the Hessian can be nicely approximated in terms of gradients, and is symmetric; both are exploited when deriving the sketched online Newton-update \citep{luo2016efficient}. We, on the other hand, have an asymmetric matrix $\Amat$. 

In the other direction, we could consider if our approach could be beneficial for the online regression setting. 
For linear regression, with $\gamma =0$, the matrix $\Amat$ actually corresponds to the Hessian. 
In contrast to previous approaches that sketched the features \citep{maillard2012linear,fard2012compressed,luo2016efficient}, 
therefore, one could instead sketch the system 
and maintain $(\Smat \Amat)^\pinv$. Since the second-order update is $\Amat^\inv \gvec_t$ for gradient $\gvec_t$
on iteration $t$, an approximate second-order update could be computed as $((\Smat\Amat)^\pinv\Smat + \eta \eye) \gvec_t$. 

In our experiments, we found sketching both sides of $\Amat$ to be less effective and found little benefit from modifying the chosen
sketch; however, these empirical conclusions warrant further investigation. With more understanding into
the properties of $\Amat$, it could be possible to benefit from this variety. For example,
sketching the left-side of $\Amat$ could be seen as sketching the eligibility trace, and the right-side as sketching
the difference between successive features. For some settings, there could be properties of either of these vectors
that are particularly suited to a certain sketch. As another example, the key benefit of many of the sketches over Gaussian random projections is in enabling the dimension $\rdim$ to be larger, by using (sparse) sketching matrices where dot product are efficient. We could not easily benefit from these properties, because $\Smat \Amat$ could be dense and computing matrix-vector products and incremental inverses would be expensive for larger $\rdim$. For sparse $\Amat$, or when $\Smat\Amat$ has specialized properties,
it could be more possible to benefit from different sketches. 

Finally, the idea of sketching fits well into a larger theme of random representations
within reinforcement learning. A seminal paper on random representations \citep{sutton1993online} demonstrates
the utility of random threshold units, as opposed to more carefully learned units. Though end-to-end
training has become more popular in recent years, there is evidence that random representations
can be quite powerful \citep{aubry2002random,rahimi2007random,rahimi2008uniform,maillard2012linear}, or even combined with descent strategies \citep{mahmood2013representation}. 
For reinforcement learning,
this learning paradigm is particularly suitable, because data cannot be observed upfront. 
Data-independent representations, such as random representations and sketching approaches,
are therefore particularly appealing and warrant further investigation for the incremental
learning setting within reinforcement learning. 


\section*{ACKNOWLEDGEMENTS}
This research was supported by NSF CRII-RI-1566186. 

{\small
\bibliography{biblio}
\bibliographystyle{icml2016}
}

\newpage
\appendix

\section{Row-rank properties of $\Smat\Amat$}\label{app_theory}
To ensure the right pseudo-inverse is well-defined in Section \ref{sec_left}, we show that the projected matrix $\Smat\Amat$ is full row-rank with high probability, if $\Amat$ has sufficiently high rank.
%
We know that the probability measure of row-rank deficient matrices for $\Smat$ has zero mass. 
However in the following, we prove a stronger and practically more useful claim that $\Smat\Amat$ is far from being row-rank deficient. Formally, we define a matrix to be \emph{$\delta$-full row-rank} if there is no row that can be replaced by another row with distance at most $\delta$ to make that matrix row-rank deficient.

\begin{proposition}\label{lem_rowrank}
Let $\Smat \in \RR^{\rdim \times \xdim}$ be any Gaussian matrix with $0$ mean and unit variance. For $r_A=rank(\Amat)$ and 
for any $\delta > 0$, $\Smat\Amat$ is $\delta$-full row-rank with probability at least $1-\exp(-2\frac{(r_A(1-0.8\delta)-\rdim)^2}{r_A})$.
\end{proposition}
\begin{proof}
Let $\Amat=\Umat\Sigmamat \Vmat^\top$ be the SVD for $\Amat$. Since $\Umat$ is an orthonormal matrix, $\Smat'=\Smat\Umat$ has the same distribution as $\Smat$ and the rank of $\Smat \Amat$ is the same as $\Smat' \Sigmamat$. Moreover notice that the last $d-r_A$ columns of $\Smat'$ get multiplied by all-zero rows of $\Sigmamat$.
Therefore, in what follows, we assume we draw a random matrix $\Smat' \in \RR^{\rdim \times r_A}$(similar to how $\Smat$ is drawn), and that $\Sigmamat \in \RR^{r_A \times r_A}$ is a full rank diagonal matrix. We study the rank of $\Smat' \Sigmamat$.



Consider iterating over the rows of $\Smat'$, the probability that any new row is $\delta$-far from being a linear combination of the previous ones is at least $1-0.8\delta$. To see why,
assume that you currently have $i$ rows and sample another vector $\vvec$ with entries sampled i.i.d. from a standard Gaussian
as the candidate for the next row in $\Smat'$. 
The length corresponding to the projection of any row $\Smat_{j:}'$ onto $\vvec$, i.e., $\Smat_{j:}' \vvec \in \RR$, is a Gaussian random variable.
Thus, the probability of the $\Smat_{j:}' \vvec$ being within $\delta$
is at most $0.8\delta$.
This follows from the fact that the area under probability density function of a standard Gaussian random variable over $[0,x]$ is at most $0.4x$, for any $x>0$. 

This stochastic process is a Bernoulli trial with success probability of at least $1-0.8\delta$. The trial stops when there are $\rdim$ successes or when the number of iterations reaches $r_A$. 
The Hoeffding inequality bounds the probability of failure by $\exp(-2\frac{(r_A(1-0.8\delta)-\rdim)^2}{r_A})$.
\end{proof}

%
%
%

\section{Alternative iterative updates}

In addition to the proposed iterative algorithm using a left-sided sketch of $\Amat$, 
we experimented with a variety of alternative updates that proved ineffective.
We list them here for completeness.

We experimented with a variety of iterative updates. For a linear system, $\Amat \wvec = \bvec$,
one can iteratively update using $\wvec_{t+1} = \wvec_t + \stepsize (\bvec - \Amat \wvec_t)$
and $\wvec_t$ will converge to a solution of the system (under some conditions). We tested the following ways to use sketched linear systems.
 
\textbf{First}, for the two-sided sketched $\Amat$, we want to solve for $\Smat_L \Amat \Smat_R^\top \wvec = \Smat_L \bvec$.
If $\tilde{\Amat}_t = \Smat_L \Amat_t \Smat_R^\top$ is square, we can use the iterative update
\begin{align*}
\tilde{\Amat}_{t+1} &= \tilde{\Amat}_t + \frac{1}{t+1}\left(\Smat_L\evec_t(\Smat_R \dvec_t)^\top - \tilde{\Amat}_t \right) \\
\tilde{\bvec}_{t+1} &= \tilde{\bvec}_t + \frac{1}{t+1}\left(r_{t+1} \Smat_L \evec_t - \tilde{\bvec}_t \right)\\
\wvec_{t+1} &= \wvec_t + \stepsize_t (\tilde{\bvec}_{t+1} -  \tilde{\Amat}_{t+1} \wvec_t)\\
&= \wvec_t + \stepsize_t (\Smat_L\bvec_{t+1} -  \Smat_L \Amat_{t+1} \Smat_R^\top \wvec_t)
\end{align*}
and use $\wvec$ for prediction on the sketched features.
Another option is to maintain the inverse incrementally, using Sherman-Morrison 
\begin{align*}
\avec_d &= \dvec_t^\top \Smat_R ^\top \tilde{\Amat}_t^\inv \\
\avec_u &=  \tilde{\Amat}_t^\inv \Smat_L \evec_t\\
\tilde{\Amat}_t^\inv &= \tilde{\Amat}_t^\inv - \frac{\avec_u \avec_d}{1+ \dvec_t^\top \avec_u} \\
\tilde{\bvec}_t &= \tilde{\bvec}_t + \frac{r_{t+1} \Smat_L \evec_t - \tilde{\bvec}_t }{t}\\
\wvec &= \tilde{\Amat}_t^\inv \tilde{\bvec}_t 
\end{align*}

If $\Smat_L \Amat \Smat_R^\top$ is not square (e.g., $\Smat_R = \eye$), we instead
solve for $\Smat_L^\top\Smat_L \Amat \Smat_R^\top\Smat_R \wvec = \Smat_L^\top\Smat_L \bvec$, where applying $\Smat_L^\top$ provides the recovery from the left and $\Smat_R$ the recovery from the right. 

\textbf{Second}, with the same sketching, we also experimented with $\Smat_L^\pinv$, instead of $\Smat_L^\top$ for the recovery, and similarly for $\Smat_R$, but this
provided no improvement. 

For this square system, the iterative update is
\begin{align*}
\wvec_{t+1} &= \wvec_t + \stepsize_t \Smat_L^\pinv (\tilde{\bvec}_{t+1} -  \tilde{\Amat}_{t+1} \Smat_R \wvec_t)\\
&= \wvec_{t+1} + \stepsize_t \Smat_L^\pinv (\Smat_L \bvec_{t+1} -  \Smat_L \Amat_{t+1} \Smat_R^\top \Smat_R \wvec_t)
\end{align*}
for the same $\tilde{\bvec}_t$ and $\tilde{\Amat}_t$ which can be efficiently kept incrementally, while the pseudoinverse of $\Smat_L$ only needs to be computed once at the beginning. 
 
\textbf{Third}, we tried to solve the system $\Smat_L^\top\Smat_L \Amat \wvec = \bvec$, using the updating rule $\wvec_{t+1} = \wvec_t + \stepsize_t (\bvec_{t+1} -  \Smat_L^\top \Smat_L \Amat_{t+1} \wvec_t)$, where the matrix $\Smat_L \Amat_{t+1}$ can be incrementally maintained at each step by using a simple rank-one update.  

\textbf{Fourth}, we tried to explicitly regularize these iterative updates by adding a small step in the direction of $\delta_t \evec_t$.

In general, none of these iterative methods performed well. We hypothesize this may be due to difficulties in choosing stepsize parameters. Ultimately, we found the sketched updated within ATD to be the most effective.

\section{Experimental details}\label{app_experiments}

Mountain Car is a classical episodic task with the goal of driving the car to the top of mountain. The state is 2-dimensional, consisting of the (position, velocity) of the car. We used the specification from \citep{sutton1998reinforcement}. We compute the true values of $2000$ states, where each testing state is sampled from a trajectory generated by the given policy. From each test state, we estimate the value---the expected return---by computing the average over 1000 returns, generated by rollouts. The policy for Mountain Car is the energy pumping policy with $20\%$ randomness starting from slightly random initial states. The discount rate is 1.0, and is 0 at the end of the episode, and the reward is always $-1$. 

Puddle World \cite{boyan1995generalization} is an episodic task, where the goal is for a robot in a continuous gridworld to reach a goal state within as fewest steps as possible. The state is 2-dimensional, consisting of ($x,y$) positions. We use the same setting as described in \citep{sutton1998reinforcement}, with a discount of 1.0 and -1 per step, except when going through a puddle that gives higher magnitude negative reward. We compute the true values from $2000$ states in the same way as Mountain Car.
A simple heuristic policy choosing the action leading to shortest Euclidean distance with $10\%$ randomness is used.

Acrobot is a four-dimensional episodic task, where the goal is to raise an arm to certain level. The reward is $-1$ for non-terminal states and $0$ for goal state, again with discount set to 1.0. We use the same tile coding as described in \citep{sutton1998reinforcement}, except that we use memory size $2^{15} = 32,768$. To get a reasonable policy, we used true-online Sarsa($\lambda$) to go through $15000$ episodes with stepsize $\alpha = 0.1/48$ and bootstrap parameter $\lambda = 0.9$. Each episode starts with a slight randomness. The policy is $\epsilon-$greedy with respect to state value and $\epsilon = 0.05$. The way we compute true values and generate training trajectories are the same as we described for the above two domains. 

Energy allocation \citep{salas2013benchmarking} is a continuing task with a five-dimensional state, where we use the same settings as in \cite{pan2017accelerated}. 
The matrix $\Amat$ was shown to have a low-rank structure \citep{pan2017accelerated} and hence matrix approximation methods are expected to perform well. 

%
For radial basis functions, we used format $k(\xvec, \cvec) = \exp(-\frac{||\xvec - \cvec||_2^2}{2\sigma^2})$ where $\sigma$ is called RBF width and $\cvec$ is a feature. On Mountain Car, because the position and velocity have different ranges, we set the bandwidth separately for each feature using $k(\xvec, \cvec) = \exp(-((\frac{\xvec_1 - \cvec_1}{0.12r_1})^2 + (\frac{\xvec_2 - \cvec_2}{0.12r_2})^2))$, where $r_1$ is the range of the first state variable and $r_2$ is the range of second state variable. 

In Figure~\ref{fig:pd_overlap_redund}, we used a relatively rarely used representation which we call spline feature. For sample $\xvec$, the $i$th spline feature is set to 1 if $||\xvec - \cvec_i|| < \delta$ and otherwise set as 0. 
The centers are selected in exactly the same way as for the RBFs. 

\paragraph{Parameter optimization.} We swept the following ranges for stepsize ($\alpha$), bootstrap parameter ($\lambda$), regularization parameter ($\eta_t$), and initialization parameter $\xi$ for all domains: 
\begin{enumerate}[nolistsep] 
	\item $\alpha \in \{0.1 \times 2.0^j | j = -7,-6, ...,4,5\}$ divided by $l_1$ norm of feature representation, 13 values in total.
	\item $\lambda \in \{0.0, 0.1, ..., 0.9, 0.93,0.95,0.97,0.99, 1.0\}$, 15 values in total.
	\item $\eta \in \{0.01 \times 2.0^j | j = -7,-6, ...,4,5\}$ divided by $l_1$ norm of feature representation, 13 values in total.
	\item $\xi \in \{10^j | j = -5, -4.25, -3.5, ..., 2.5, 3.25, 4.0\}$, 13 values in total. 
\end{enumerate}
To choose the best parameter setting for each algorithm, we used the sum of RMSE across all steps for all the domains Energy allocation. For this domain, optimizing based on the whole range causes TD to pick an aggressive step-size to improve early learning at the expense of later learning. Therefore, for Energy allocation, we instead select the best parameters based on the sum of the RMSE for the second half of the steps.

For the ATD algorithms, as done in the original paper, we set $\alpha_t = \frac{1}{t}$ and only swept the regularization parameter $\eta$,
which can also be thought of a (smaller) final step-size. For this reason, the range of $\eta$ is set to $0.1$ times the range of $\alpha$,
to adjust this final stepsize range to an order of magnitude lower. 

\paragraph{Additional experimental results.} 
In the main paper, we demonstrated a subset of the results to highlight conclusions. For example, we showed the learning curves and parameter sensitivity in Mountain Car, for RBFs and tile coding. Due to space, we did not show the corresponding results for Puddle World in the main paper; we include these experiments here. Similarly, we only showed Acrobot with RBFs in the main text, and include results with tile coding here.

\begin{figure*}[htp!]
		\subfigure[Puddle World, tile coding,  $\rdim = 50$]{
			\includegraphics[width=\figwidththree]{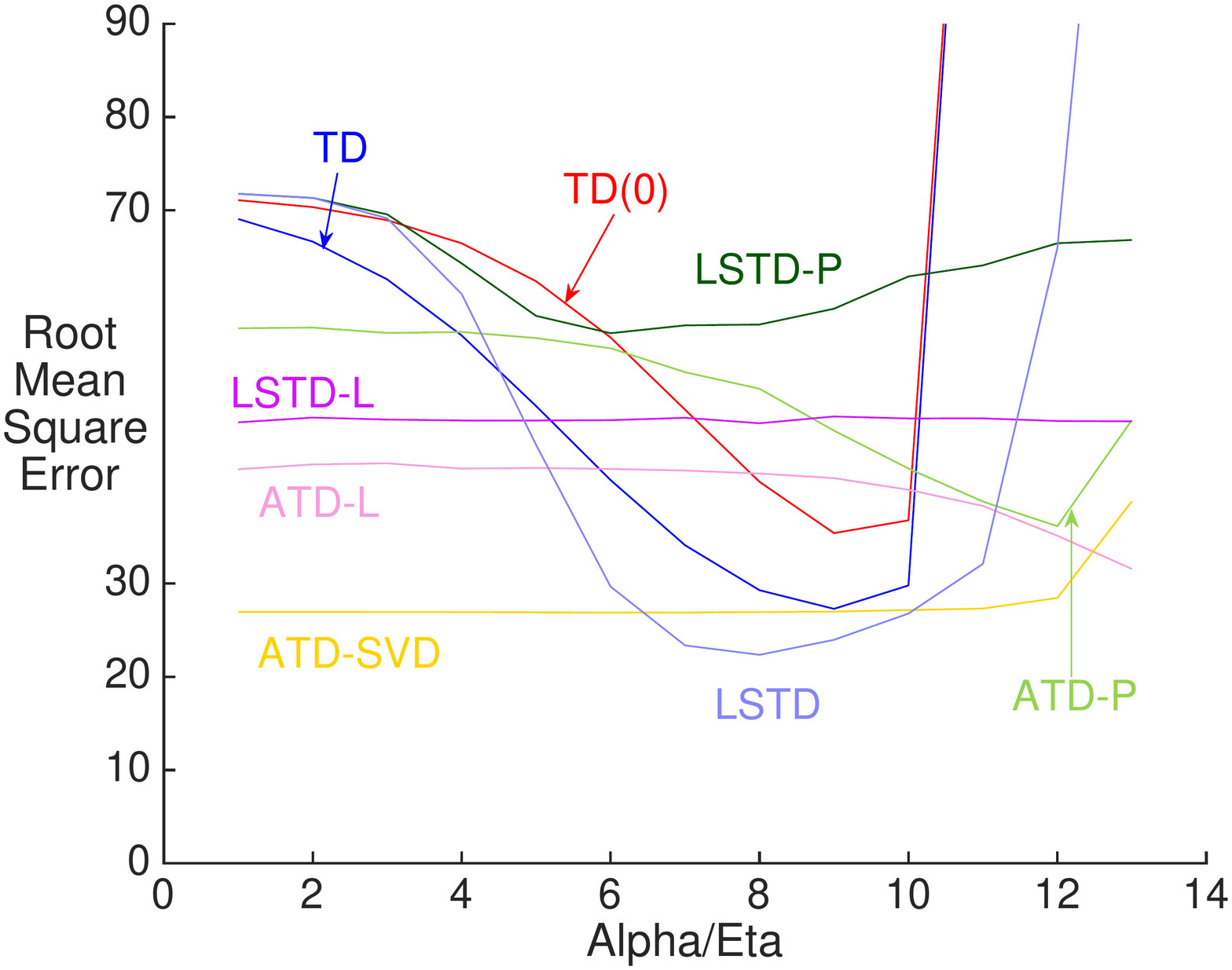} \label{fig:pd_gau_tile50r}}
		\subfigure[Puddle World, RBF, $\rdim = 50$]{
			\includegraphics[width=\figwidththree]{figures/pd_gau_rbf50r.pdf}} \\
	\subfigure[Puddle World, tile coding, $\rdim = 50$]{
		\includegraphics[width=\figwidththree]{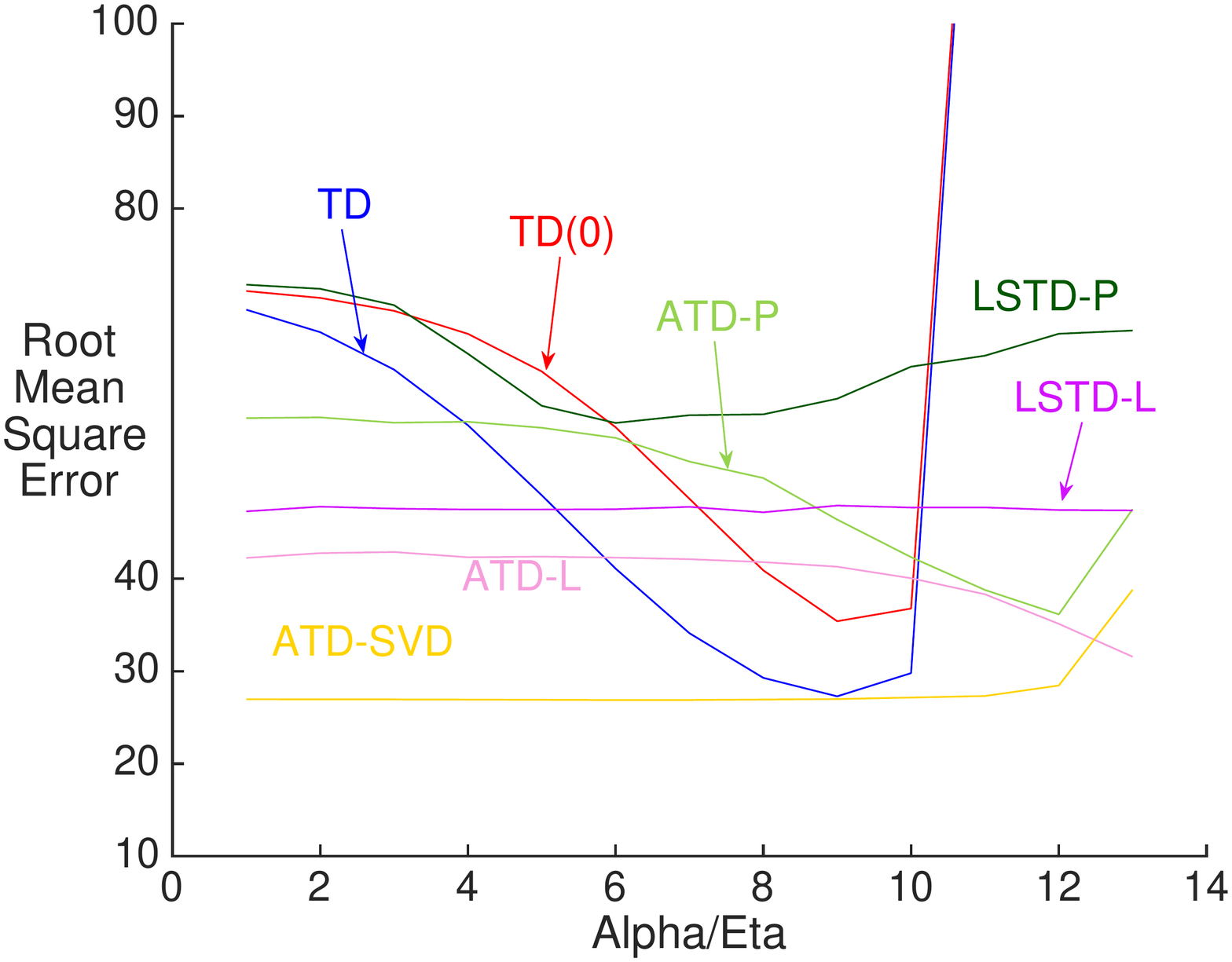}\label{fig:pd_gau_tile50rsensi}}	
	\subfigure[Puddle World, RBF, $\rdim = 50$]{
		\includegraphics[width=\figwidththree]{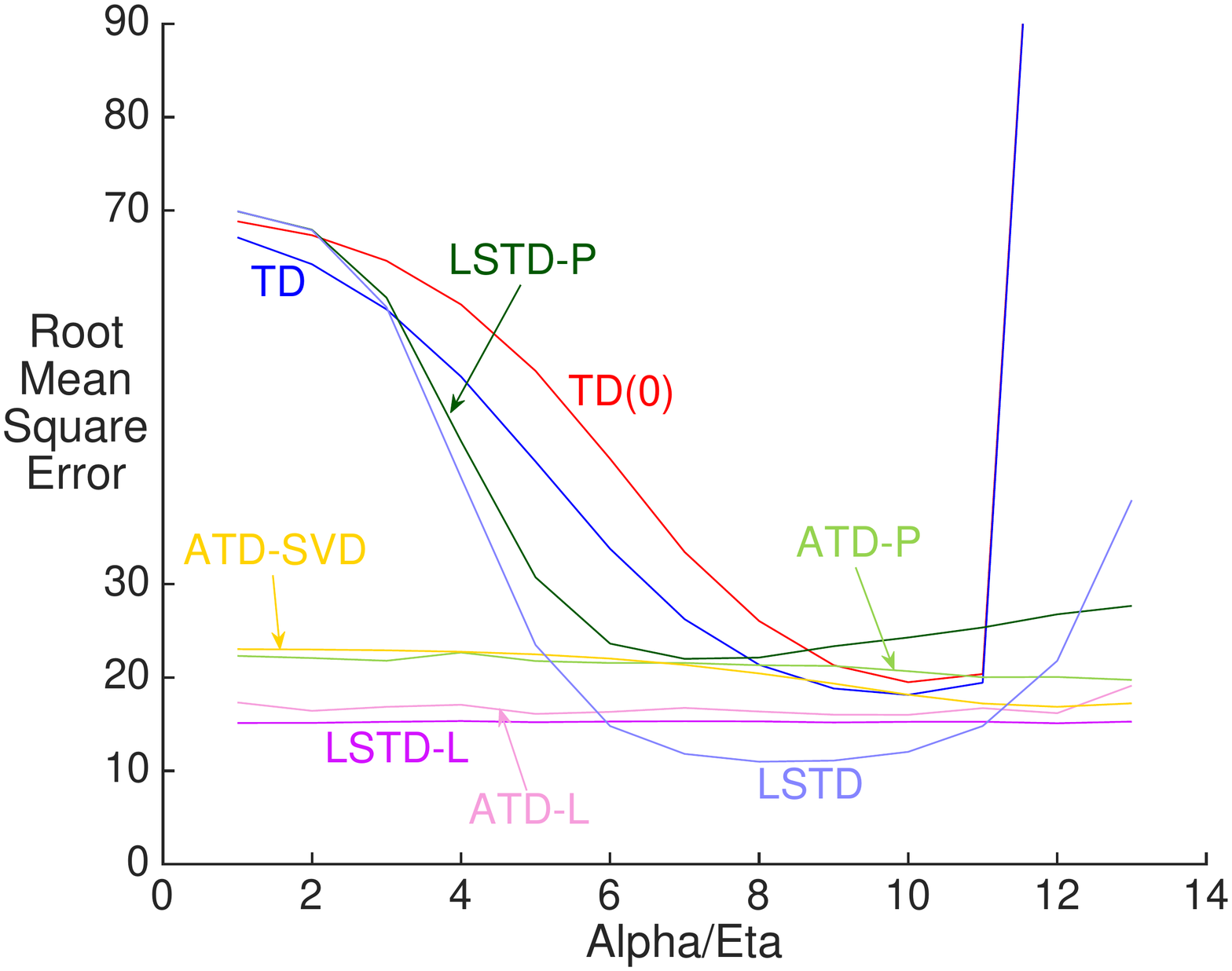}\label{fig:pd_gau_rbf50rsensi}}	
	\begin{minipage}{\figwidththree}
	\vspace{-8.0cm}
	\caption{ 
		The two sensitivity figures are corresponding to the above two learning curves on Puddle World domain. Note that we sweep initialization for LSTD-P, but keep initialization parameter fixed across all other settings. The one-side projection is almost insensitive to initialization and the corresponding ATD version is insensitive to regularization. Though ATD-SVD also shows insensitivity, performance of ATD-SVD is much worse than sketching methods for the RBF representation. And, one should note that ATD-SVD is also much slower as shown in the below figures. 
	}\label{fig_pd_mcar_sensi}
	\end{minipage}
\end{figure*}

\begin{figure*}[htp!]
	\centering
	\subfigure[Mountain Car, Tile coding, $\rdim = 25$]{
		\includegraphics[width=\figwidththree]{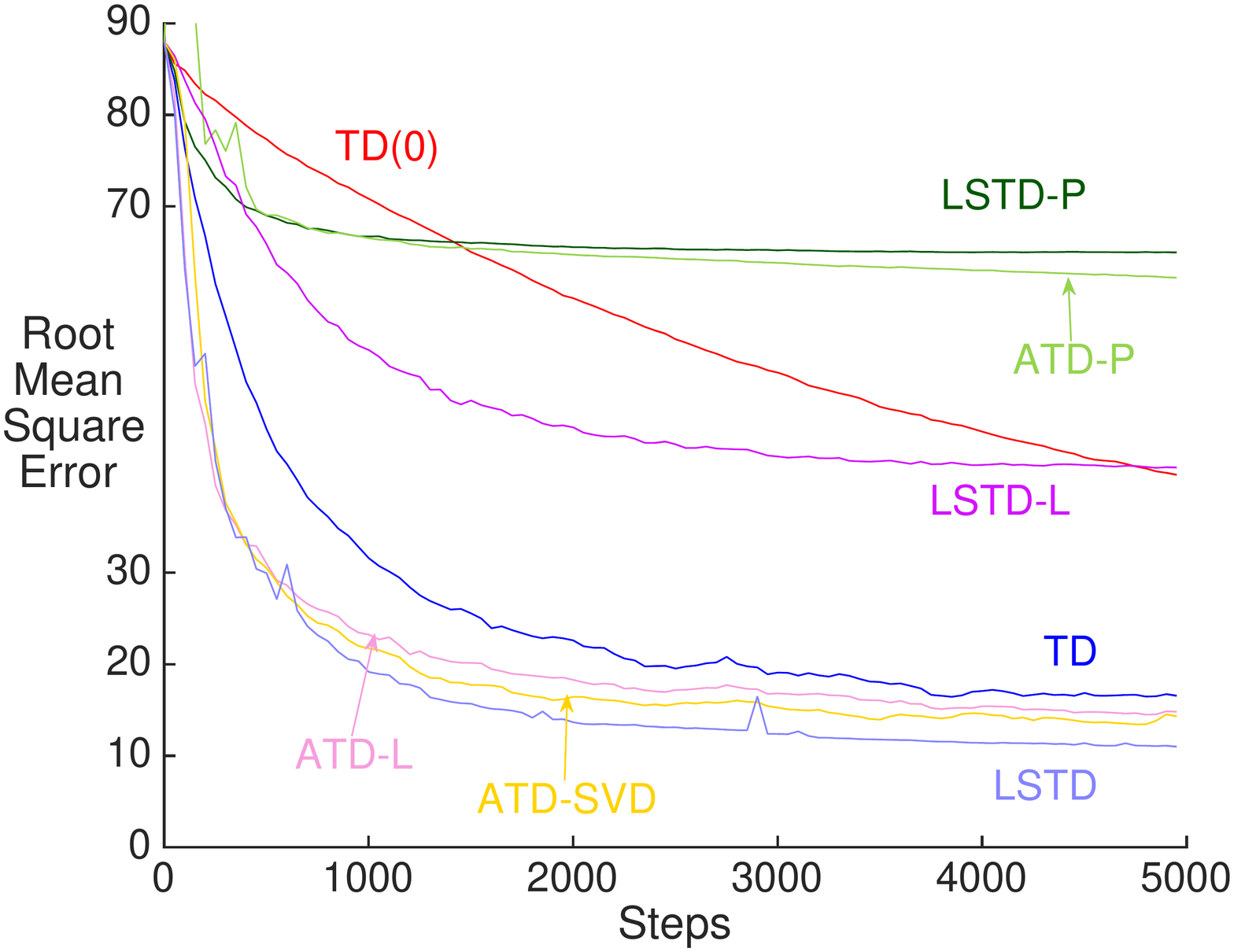} \label{fig:mcar_gau_tile25r}} 
	\subfigure[Mountain Car, Tile coding, $\rdim = 50$]{
		\includegraphics[width=\figwidththree]{figures/mcar_gau_tile50r.pdf} 
		}
	\subfigure[Mountain Car, Tile coding, $\rdim = 75$]{
		\includegraphics[width=\figwidththree]{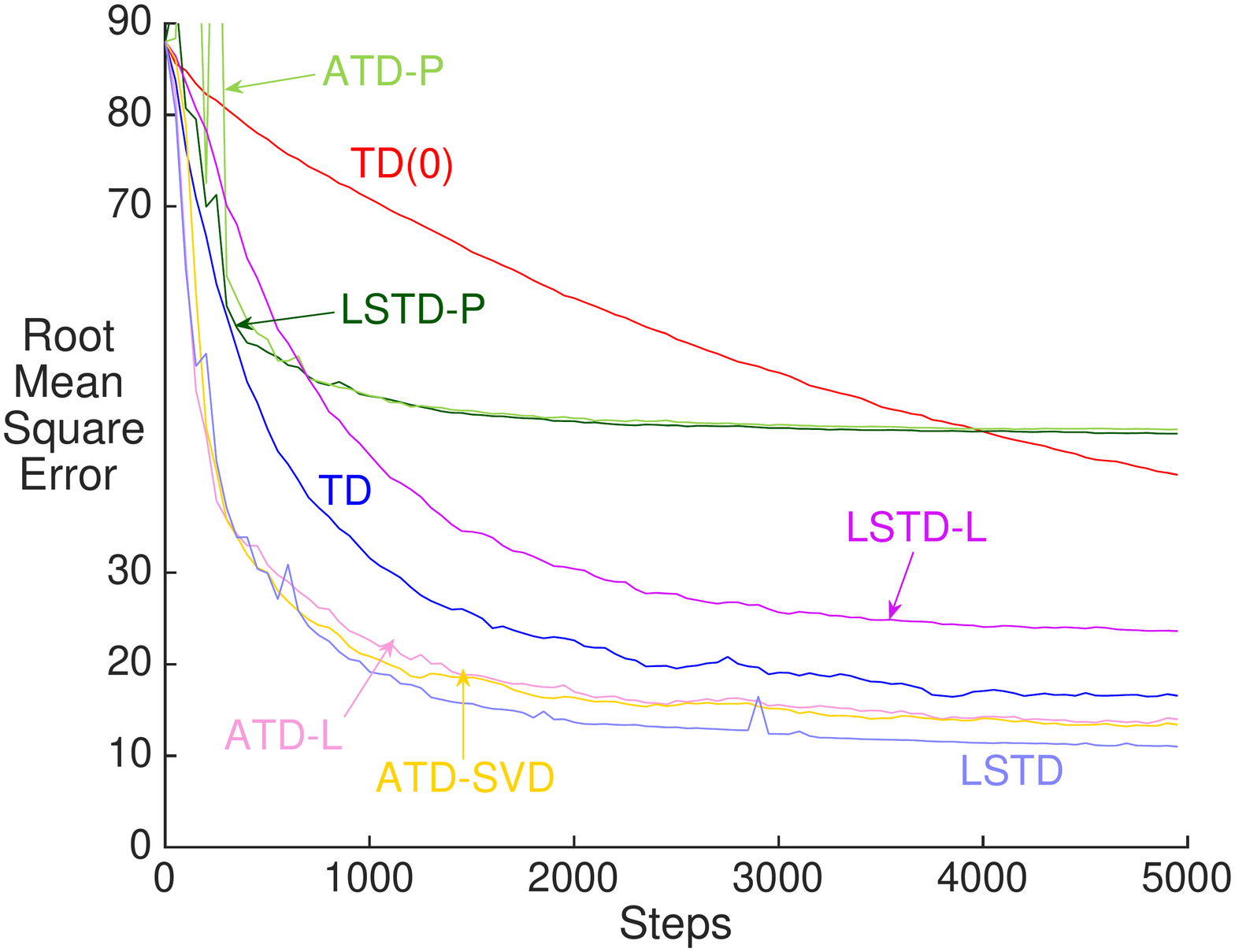}\label{fig:mcar_gau_tile75r}}
	\caption{ 
		Change in performance when increasing $\rdim$, from $25$ to $75$.  We can draw similar conclusions to the same experiments in Puddle World in the main text. Here, the unbiased of ATD-L is even more evident; even with as low a dimension as 25, it performs similarly to LSTD. 
	}\label{fig:mcar_tile_rank_lc}
\end{figure*}

\begin{figure*}[htp!]
	\centering
	\subfigure[Acrobot, tile coding, $\rdim = 50$]{
		\includegraphics[width=\figwidththree]{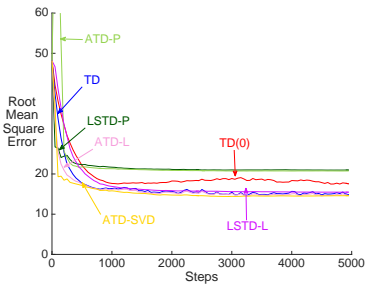} \label{fig:acro_gau_tile50r}} 
	\subfigure[Acrobot, tile coding, $\rdim = 50$]{
		\includegraphics[width=\figwidththree]{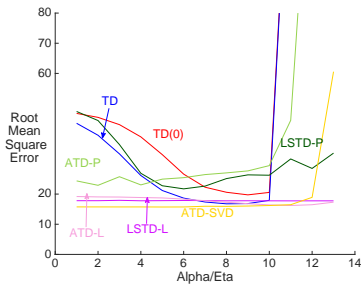} 
	}
	\subfigure[Acrobot, RBF, $\rdim = 75$]{
		\includegraphics[width=\figwidththree]{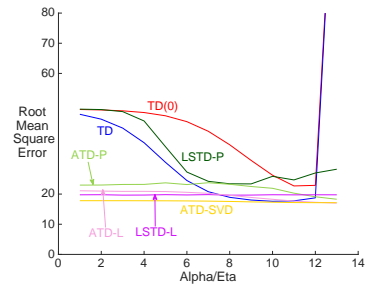}\label{fig:acro_gau_rbf75r}}
	\caption{ 
	Additional experiments in Acrobot, for tile coding with $\rdim = 50$ and for RBFs with $\rdim = 75$. 
	}\label{fig:mcar_rank_lc}
\end{figure*}

\end{document}